\documentclass{aims} 
\usepackage{amsmath}
\usepackage{paralist}
\usepackage[misc]{ifsym}
\usepackage{epsfig} 
\usepackage{epstopdf} 
\usepackage[colorlinks=true]{hyperref}
\usepackage{amsfonts}
\usepackage{amsmath, cases}
\usepackage{indentfirst}
\usepackage{graphicx}
\usepackage{subfigure}
\usepackage{latexsym,bm, amsthm}
\usepackage{multirow}
\usepackage{epsfig}
\usepackage{latexsym}
\usepackage{float}
\usepackage{amssymb}
\usepackage{algorithm}
\usepackage{algorithmic}
\usepackage{cleveref}  
\usepackage{epstopdf} 
\usepackage{float}
\usepackage[normalem]{ulem}
\usepackage{tabularray}
\usepackage{natbib}
\setcitestyle{numbers,square}

\usepackage{float}
\newcommand{\iq}{\boldsymbol{i}}
\newcommand{\jq}{\boldsymbol{j}}
\newcommand{\kq}{\boldsymbol{k}}

\newcommand{\pq}{\boldsymbol{p}}

\newcommand{\qq}{\boldsymbol{q}}

\newcommand{\Bq}{\boldsymbol{B}}
\newcommand{\bq}{\boldsymbol{b}}

\newcommand{\Dq}{\boldsymbol{D}}

\newcommand{\Sq}{\boldsymbol{S}}
\newcommand{\Xq}{\boldsymbol{X}}
\newcommand{\xq}{\boldsymbol{x}}

\newcommand{\zq}{\boldsymbol{z}}

\newcommand{\yq}{\boldsymbol{y}}

\newcommand{\gq}{\boldsymbol{g}}
\newcommand{\Sb}{\boldsymbol{S}}
\newcommand{\alphaq}{\boldsymbol{\alpha}}
\newcommand{\zeroq}{\boldsymbol{0}}

\hypersetup{urlcolor=blue, citecolor=red}
\allowdisplaybreaks

\def\currentvolume{X}
 \def\currentissue{X}
  \def\currentyear{200X}
   \def\currentmonth{XX}
    \def\ppages{X-XX}
     \def\DOI{10.3934/xx.xxxxxxx}

\newtheorem{theorem}{Theorem}[section]

\newtheorem{lemma}[theorem]{Lemma}

\theoremstyle{definition}
\newtheorem{definition}[theorem]{Definition}

\title[Wasserstein Quaternion Generative Adversarial Network]
{A Novel Wasserstein Quaternion Generative Adversarial Network for Color Image Generation} 

\author[Z. Jia, D. Wang, H. Wang, Y. Xie, M. Zhao and X. Zhao]{}


\thanks{This work is  supported  in part by the  National Key R\&D program of China (No. 2023YFA1010101); the National Natural Science Foundation of China under grants 12171210, 12090011, 12371378,    and 62402286;  the “QingLan” Project for Colleges and Universities of Jiangsu Province (Young and middle-aged academic leaders);		
	the Major Projects of Universities in Jiangsu Province (No. 21KJA110001); the Natural Science Foundation of Fujian Province of China grants 2022J01378 and 2023J011127; the China Postdoctoral Science Foundation with code 2024T170510; the Natural Science Foundation of Shandong Province with code ZR2023QA059; and the  Postgraduate
	Research \& Practice Innovation Program of Jiangsu Province KYCX23$\_$2900. }

\thanks{$^*$Corresponding author: Duan Wang (E-mail: duanw@jsnu.edu.cn)}

\begin{document}
\maketitle

\begin{center}
  \scshape
  Zhigang Jia$^{{\href{zhgjia@jsnu.edu.cn}{\textrm{\Letter}}}1,4}$,
  Duan Wang$^{{\href{duanw@jsnu.edu.cn}{\textrm{\Letter}}}*1}$,
  Hengkai Wang$^{{\href{2020241288@jsnu.edu.cn}{\textrm{\Letter}}}1}$,
  Yajun Xie$^{{\href{xyj@fzfu.edu.cn}{\textrm{\Letter}}}3}$,  \\ 
  Meixiang Zhao$^{{\href{zhaomeixiang2008@126.com}{\textrm{\Letter}}}1}$
  and Xiaoyu Zhao$^{{\href{ustcxyz@hotmail.com}{\textrm{\Letter}}}2}$
  
\end{center}
\medskip

{\footnotesize
 \centerline{$^1$School of Mathematics and Statistics,\\ Jiangsu Normal University, Xuzhou 221116, P. R. China}{}
} 

\medskip

{\footnotesize
	\centerline{$^2$The School of Mathematics, Hefei University of Technology, Hefei 230000, P.R. China}
}

\bigskip

{\footnotesize
	\centerline{$^3$The School of Big Data,  Fuzhou University of International Studies and Trade, Fuzhou 350202, P. R. China}
}
\medskip

{\footnotesize
 \centerline{$^4$The Research Institute of Mathematical Science,  Jiangsu Normal University,  Xuzhou 221116, P. R. China}
}
\medskip



\begin{abstract}
Color image generation has a wide range of applications, but the existing generation models ignore the correlation among color channels, which may lead to chromatic aberration problems. In addition, the data distribution problem of color images has not been systematically elaborated and explained, so that there is still the lack of the theory about measuring different color images datasets. In this paper, we define a new quaternion Wasserstein distance and develop its dual theory.  To deal with the quaternion linear programming problem, we derive the strong duality form with helps of quaternion convex set separation theorem and quaternion Farkas lemma.  With using quaternion Wasserstein distance, we propose a novel Wasserstein quaternion generative adversarial network. Experiments demonstrate that this novel model surpasses both the (quaternion) generative adversarial networks and the Wasserstein generative adversarial network in terms of generation efficiency and image quality.
\end{abstract}


\section{Introduction}
Color image generation models can be utilized for tasks such as image inpainting, denoising, and style transfer. Quaternion generative adversarial networks (QGANs) have attracted attention for their ability to preserve relationships among color image channels. However, using Jensen-Shannon divergence as an evaluation metric may not effectively quantify the relationship between real and generated images. Wasserstein distance serves as a robust measure of assessing distances between two distributions. Nevertheless, the theory of quaternion Wasserstein distance (QWD) is still in the early stages of development. This paper aims to introduce a novel Wasserstein quaternion generative adversarial network (WQGAN) model by deriving QWD and its dual form.

Generative adversarial networks (GAN) have been put into various practical applications as a generative model \cite{r15}. The training stability of GAN has been an difficult challenge to overcome since it was introduced in 2014 \cite{r4}. Many methods are proposed to improve the generative ability of GAN. Deep networks are first introduced into the GAN architecture with the proposition of deep convolutional generative adversarial networks (DCGAN) \cite{r5}. DCGAN improves the stability of the model and lays the foundation for the development of GANs. However, its model still suffers from mode collapse and needs further improvement. The authors systematically demonstrate what is the problem with the original GAN \cite{r6} and propose the Wasserstein generative adversarial networks (WGAN) \cite{r7}. The introduction of Wasserstein distance provides research direction for later GANs. 

The Wasserstein theory is a core component of the latest Optimal Transport (OT) theory. Recently, OT, as an emerging mathematical tool, has brought new ideas and solutions to the field of machine learning, gradually becoming a research hotspot.
The core idea of optimal transport originated from a classic mathematical problem, aiming to find an optimal way to transport the mass in one probability distribution to another while minimizing the transportation cost \cite{monge1781}.
Its classic formulation is the Kantorovich problem \cite{Kantor1942}, which precisely measures the transportation cost from one point to another through a carefully defined cost function.
To improve problem solvability and transportation plan characteristics, entropy regularization \cite{comot2018} is introduced to make transportation smoother and more uniform, avoiding over-concentration. As a core metric of optimal transport, the Wasserstein distance can accurately characterize distribution differences \cite{wdisapp2018}, laying a foundation for subsequent applications. In supervised learning, optimal transport can serve as an innovative loss function to capture subtle distribution differences or combine with cross-entropy loss to enhance model performance, and it also ensures model fairness through data adjustment and optimizes evaluation indicators \cite{Fairnessot2019, Fairnessot2020, Fairnessmodelot2020}. 
In unsupervised learning, optimal transport supports the design of encoding and decoding in Variational Autoencoders \cite{vae2020, vaeot2018}. WGAN uses it to replace traditional divergence for optimizing GAN performance \cite{gan2ot2017}, then SNGAN \cite{r8} and BigGAN further improve training stability based on this. However, these generative models ignore inter-channel correlations when processing color images, easily leading to color deviations.

Quaternion is an excellent tool used in many  methods of image processing. 
Two new nonlocal self-similarity (NSS)-based quaternion matrix completion (QMC) and quaternion tensor completion (QTC) algorithms are presented to deal with the color image problems in  \cite{s24}, where nonlocal self-similarity technique is introduced to reduce the low-rank prior requirement. To preserve the structure of color channels when restoring color images, a quaternion-based weighted nuclear norm minimization (WNNM) method is proposed \cite{zengtieyong_q}. Furthermore, a novel quaternion-based weighted schatten p-norm minimization (WSNM) model is proposed for tackling color image restoration problems with preliminary theoretical convergence analysis \cite{liu_q}.
A novel method named QLNM-QQR based on quaternion Qatar Riyal decomposition  and quaternion $L_{1,2}$-norm is proposed to reduces computational complexity by avoiding the need for calculating the QSVD of large quaternion matrices \cite{jifei_q}.
Then a new method called quaternion nuclear norm minus Frobenius norm minimization is employed into color image reconstruction by capturing  RGB channels relationships comprehensively in \cite{k6}. For color image inpainting, a quaternion matrix completion method through untrained QCNN is proposed  \cite{k7}. Quaternion circulant matrix can be block-diagonalized into 1-by-1 block and 2-by-2 block matrices by permuted discrete quaternion Fourier transform matrix, this shows that the inverse of a quaternion cyclic matrix can be determined quickly and efficiently \cite{Mg_q}, laying the groundwork for the rest of the image study. 

Quaternion was first introduced into GAN and applied to color image generation in 2021 \cite{r2}. 
Then authors  \cite{r11} directly applied quaternion convolution to SNGAN in order to produce high-quality color images. But none of these directly address the challenge of expanding from low-dimensional data to high-dimensional data in the quaternion domain. Until 2024,  a novel quaternion deconvolution operation was proposed to construct a new quaternion generative adversarial networks (QGAN) in \cite{r12}. This QGAN improves the stability of initial training of the generative model and is applied to color image inpainting. However, the loss function utilized by these quaternion generative adversarial networks struggles to effectively measure the distributional relationship between datasets, which can not guarantee the ultimate stability of the model. 

This study defines a novel QWD to address issues of existing color image generation models, such as ignoring correlations among color channels and lacking theories for measuring distribution differences between datasets. By leveraging the quaternion convex set separation theorem and quaternion Farkas lemma, the study derives QWD's dual theory and the strong duality form of quaternion linear programming problems, solving the non-differentiability and high computational complexity of traditional distribution distance metrics in deep learning frameworks. QWD can directly characterize distribution differences between color images represented by quaternions, laying a foundation for constructing WQGAN.

\begin{figure}[ht]
	\centering
	\setlength{\abovecaptionskip}{-0cm}
	\setlength{\belowcaptionskip}{-0cm}
	\vspace{-0.1cm}
	\includegraphics[width=1\textwidth,height=0.6\textwidth]{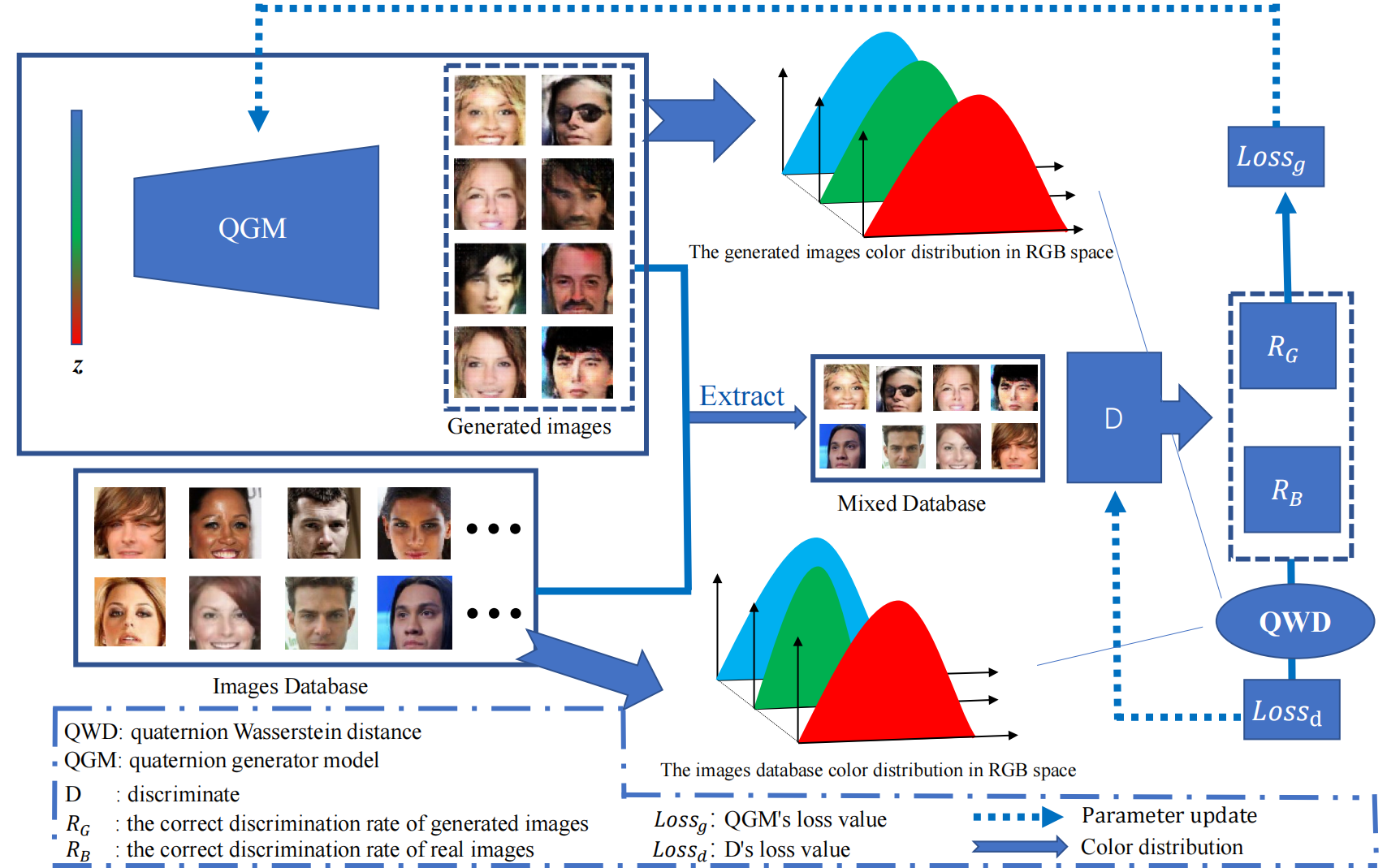}
	\caption{The architecture of WQGAN.}
	\label{p:WQGANarchitecture}
\end{figure}

The architecture of WQGAN is given in Figure \ref{p:WQGANarchitecture}.  Compared to QGAN, we change the training objective function of the model from cross-entropy function to QWD function, which effectively improves the generative ability of the image generation model.
The main contribution is in three folds:
\begin{itemize}
	\item QWD is a new breakthrough in applying quaternions to probability distributions. First we formulate the quaternion linear programming problem and give its quaternioni dual form. Then we give a dual form of the discrete form of QWD. 
	This provides a deeper theoretical foundation for the application of Wasserstein distance in quaternion deep learning networks.
	
	\item  A novel Wasserstein quaternion generative adversarial network is proposed based on the dual form of QWD. Through mathematical theory, we give a measuring method for quaternion probability distribution about two color image datasets. Then propose a novel quaternion adversarial generative network model using dual form of QWD.
	
	\item The experiment results show that WQGAN can generate high quality images faster compared to QGANs as well as WGAN. The images generated by WQGAN at 50k iterations have an FID score $20.1314$ lower than WGAN and $12.2806$ lower than QGAN.   
\end{itemize}

The rest of the paper is organized as follows. In Section \ref{sec:pre}, we recall preliminary work about quaternion and distance function. In Section \ref{sec:dual}, we derive a quaternion dual form, encompassing both weak duality and strong duality. In Section \ref{sec:wqgan}, we propose a novel WQGAN model by deriving quaternion measure integrals and the corresponding cost functions. In section \ref{sec:exs}, we demonstrate numerical examples to illustrate the superiority of proposed model. In Section \ref{sec:con}, we present conclusion.

\section{Preliminaries} \label{sec:pre}
In this section, we present the definitions of   quaternion random variable and quaternion Wasserstain distance. 
\subsection{Quaternion and quaternion random variables}
The quaternion, which extends real and complex  numbers naturally, was first introduced in 1843 \cite{r1}. 
Let $\mathbb{N}_{+}$, $\mathbb{R}$ and $\mathbb{Q}$ denote the sets of positive integers,  real numbers and quaternions, respectively. 
A quaternion $\qq \in \mathbb{Q} $ is defined as  
$
\qq=q^{(0)}+q^{(1)}\iq+q^{(2)}\jq+q^{(3)}\kq,
$
$q^{(j)}\in\mathbb{R}~ (j = 0,1,2,3)$.   Three imaginary units $\iq, ~\jq$ and $\kq$  satisfy
$
\iq^{2}=\jq^{2}=\kq^{2}=\iq\jq\kq=-1.
$
The absolute value of $\qq$ is defined by 
$|\qq|=\sqrt{|q^{(0)}|^2+|q^{(1)}|^2+|q^{(2)}|^2+|q^{(3)}|^2}.$
If $q^{(0)}=0$, then such $\qq$ is called a purely imaginary quaternion. If $q^{(j)}>0~ (j = 0,1,2,3)$, then $\qq$ is called positive quaternion, denoted by $\qq>0$.   If $q^{(j)}\ge 0~ (j = 0,1,2,3)$, then $\qq$ is called nonnegative  quaternion, denoted by $\qq\ge 0$. We say two quaternions $\qq_1$ and $\qq_2$ satisfy $\qq_1 >\qq_2$ ($\qq_1\ge\qq_2$) if $\qq_1-\qq_2> 0$ ($\qq_1-\qq_2\ge 0$).

Let $\mathbb{R}^n$ and $\mathbb{Q}^n$ with  $n\in\mathbb{N}_{+}$ denote the $n$-dimensional linear spaces of real and quaternion vectors, respectively. 
Referring to \cite{r2}, we call $\Xq$  a quaternion $\mathbb{Q}^n$-valued random variable if $\Xq=X^{(0)}+X^{(1)}\iq+X^{(2)}\jq+X^{(3)}\kq$, where $X^{(j)}=[X_i^{(j)}]_{i=1}^n$ $(j = 0,1,2,3)$ are $\mathbb{R}^n$-valued random variables which are pairwise independent.
Recall that the Euclidean norm of real vector $X^{(j)}$ is $\|X^{(j)}\|=\sqrt{\sum_{i=1}^n|X_i^{(j)}|^2}$.
The Euclidean norm of quaternion vector $\Xq$ is  defined by
$$\|\Xq\|=\sqrt{\|X^{(0)}\|^2+\|X^{(1)}\|^2+\|X^{(2)}\|^2+\|X^{(3)}\|^2}:=\sqrt{\sum\limits_{i=1}^n\sum\limits_{j=0}^3|X_i^{(j)}|^2}.$$
If each entry of $\Xq$ is positive, then $\Xq$ is called positive quaternion vector, denoted by $\Xq>0$. If each entry of $\Xq$ is nonnegative, then $\Xq$ is called nonnegative quaternion vector, denoted by $\Xq\ge 0$. 
We say two quaternion vectors $\Xq_1$ and $\Xq_2$ satisfy $\Xq_1>\Xq_2$ ($\Xq_1\ge\Xq_2$) if $\Xq_1-\Xq_2> 0$ ($\Xq_1-\Xq_2\ge 0$). 

\begin{definition}\label{def:pfd}
	A quaternion $\mathbb{Q}^n$-valued random variable $\Xq=X^{(0)}+X^{(1)}\iq+X^{(2)}\jq+X^{(3)}\kq$ is called discrete random variable if $X^{(j)}\in\mathbb{R}^n~(j = 0,1,2,3)$ are discrete $\mathbb{R}^n$-valued random variables. Let $p(X^{(j)})$ be the probability mass function of $X^{(j)}$, then the probability mass function (p.m.f. for short) of $\Xq$ is defined by:
	\begin{equation}
		\pq(\Xq)= p(X^{(0)}) \times p(X^{(1)}) \times p(X^{(2)}) \times p(X^{(3)}).
	\end{equation}
	We call the set of  $\Xq$  satisfying $\pq(\Xq)>0$ the support of $\pq$. 
\end{definition}

\subsection{Quaternion Wasserstein distance}	
In previous QGANs  \cite{r12}, the loss function employed was the quaternion cross-entropy loss function. This loss function quantifies the dissimilarity between two quaternion distributions by utilizing the Kullback-Leibler (KL)  and Jensen-Shannon (JS) divergences:
\begin{equation}
	KL(P_r||P_g)=\sum_{i=1}^{n} p(\xq_i)\log\frac{p(\xq_i)}{q(\xq_i)},
\end{equation}
\begin{equation}
	JS(P_r||P_g)=\frac{1}{2}\left( KL(P_r||M) + KL(P_g||M)\right), M = \frac{1}{2}\left(P_r+P_g\right),
\end{equation}
where $P_r,P_g$ are two probability distributions and $p(\xq), q(\xq)$ are corresponding probability mass functions.

When two probability distributions are very far apart with no overlap, the KL divergence becomes meaningless, while the JS divergence yields a constant value $\log2$ \footnote{$JS(P_r||P_g)=\frac{1}{2}\left( KL(P_r||P_g) + KL(P_g||P_r)\right)=\frac{1}{2}\sum p(\xq)\log\left(\frac{p(\xq)}{p(\xq)+q(\xq)}\right)+\frac{1}{2}\sum q(\xq)\log\left(\frac{q(\xq)}{p(\xq)+q(\xq)}\right)+\log2$. When $p(\xq)\rightarrow0, q(\xq)\rightarrow1,$ $JS(P_r||P_g)=\log2.$}. In learning algorithms, this can result in a gradient of zero at that point. This phenomenon leads to the vanishing gradient problem about learning \cite{r4}.

For any two quaternion vectors $\qq_1,\qq_2\in\mathbb{Q}^n$, i.e., $\qq_1=q_1^{(0)}+q_1^{(1)}\iq+q_1^{(2)}\jq+q_1^{(3)}\kq,\qq_2=q_2^{(0)}+q_2^{(1)}\iq+q_2^{(2)}\jq+q_2^{(3)}\kq$, where $q_i^{(j)}\in \mathbb{R}^n$ $(i=1,2;j=0,1,2,3)$, we define the distance between $\qq_1$ and $\qq_2$ by 
$$\|\qq_1-\qq_2\|=\sqrt{\sum\limits_{j=0}^3\left\|q_1^{(j)}-q_2^{(j)}\right\|^2}.$$

Let $\mathcal{B}(\mathbb{Q}^n)$ be the Borel set on $\mathbb{Q}^n$ induced by the distance $\|\cdot\|$. Meanwhile, let $\mathcal{M}(\mathbb{Q}^n)$ be the set of all probability measures on Borel set $\mathcal{B}(\mathbb{Q}^n)$. Then the QWD is defined as 
\begin{equation}\label{defw0}
	\mathcal{W}(P_r,P_g)=\inf_{\Phi\in\Pi(P_r,P_g)}\int\int_{\mathbb{Q}^n\times \mathbb{Q}^n}\|u-v\|\Phi(d u,d v)=\inf_{\Phi\in\Pi(P_r,P_g)}\mathbb{E}_{(\xq,\yq)\sim\Phi}[\|\xq-\yq\|], 
\end{equation}
for all $P_r,P_g \in  \mathcal{M}(\mathbb{Q}^n)$, 
where $\Pi(P_r,P_g)$ is set of all couplings measures $\Phi$ of $P_r$ and $P_g$, i.e.,
\begin{equation}\label{cons1}
	\Phi(A\times \mathbb{Q}^n)=P_r(A),\quad \Phi(\mathbb{Q}^n\times A )=P_g(A),\quad \forall\, A\in \mathcal{B}(\mathbb{Q}^n).
\end{equation}
In terms of probability mass functions, let $\pq_r$, $\pq_g$ and $\gamma$ be the probability mass functions corresponding to $P_r$, $P_g$ and  $\Phi$, respectively, then
equation \eqref{defw0} becomes
\begin{equation}\label{defw}
	\mathcal{W}(P_r,P_g)=\inf_{\gamma\in\Pi(P_r,P_g)}
	\left\{
	\sum_{\xq\in S_r}\sum_{\yq \in S_g} \Vert \xq-\yq\Vert\gamma(\xq,\yq)\right\},
\end{equation}
where $S_r$ and $S_g$ are the supports of $\pq_r$ and $\pq_g$, respectively.
The constraint \eqref{cons1} can be rewritten as
\begin{equation}\label{cons2}
	\begin{split}
		\sum_{\yq\in S_g } \gamma(\xq,\yq)=\pq_r(\xq),\quad \forall \, \xq\in \mathbb{Q}^n,\\
		\sum_{\xq\in S_r } \gamma(\xq,\yq)=\pq_g(\yq),\quad \forall \, \yq\in \mathbb{Q}^n.
	\end{split}
\end{equation}

We remark that  the computation of the original QWD necessitates evaluating the distance between the generated distribution and the real data distribution throughout the entire data space. However, this is not feasible in practice due to the high dimensionality and continuity of data spaces, which renders the computation of such distances exceedingly intricate.
Additionally, the original QWD is generally non-differentiable, posing challenges in applying backpropagation algorithms for training within deep learning frameworks. 
Therefore, we need to propose a discrete form of QWD and derive its dual form (in Section \ref{sec:dual}).  Applying this dual form, we are able to  construct a new WQGAN (in Section \ref{sec:wqgan}).

\section{Quaternion linear programming problem and  dual form} \label{sec:dual}
In this section, we  present a quaternion linear programming (QLP) problem for QWD and derive its dual form. 

The QLP problem we aim to address  is
\begin{equation} \label{e:lpp}
	\min_{\boldsymbol{\Gamma}\in \mathbb{Q}^{m}}\left\{|C^\top\boldsymbol{\Gamma}|\mid \Upsilon\boldsymbol{\Gamma}=\bq,~\boldsymbol{\Gamma} \geq \zeroq \right\}, \boldsymbol{b} \in \mathbb{Q}^{n}, C\in \mathbb{R}_+^{m}, \Upsilon\in \mathbb{R}^{n \times m}.
\end{equation}
Here, the notation $^T$ denotes the transpose operator and $\mathbb{R}^{n \times m}$
the set of all $n\times m$ real matrices.

At first,  we present a weak dual form of this QLP problem. Let $\mathbb{R}_+^n$ denote the set of all nonnegative $n$-dimensional real vectors. 
\begin{theorem}
	\label{th_dual}
	Given  quaternion vector $\boldsymbol{b} \in \mathbb{Q}^{n}$, nonnegative real vector $C\in \mathbb{R}_+^{m}$ and real matrix $ \Upsilon\in \mathbb{R}^{n \times m}$, there is
	\begin{equation}\label{e:w-dual}
		\max_{y\in \mathbb{R}^{n}}\left\{|\bq^\top y|\mid \Upsilon^\top y \leq C,~ \bq^\top y \geq \zeroq\right\}\leq\min_{\boldsymbol{\Gamma\in \mathbb{Q}^{m}}}\left\{|C^\top\boldsymbol{\Gamma}|\mid \Upsilon\boldsymbol{\Gamma}=\bq,~\boldsymbol{\Gamma}\geq\zeroq\right\}.
	\end{equation}
\end{theorem}

\begin{proof}
	Suppose $\hat{\boldsymbol{\Gamma}}$ is a solution of the QLP problem \eqref{e:lpp}. That is, $\Upsilon\hat{\boldsymbol{\Gamma}}=\boldsymbol{b}.$  
	Then for any $y\in \mathbb{R}^{m}$ satisfying $y^{\top}\Upsilon \leq C^{\top}$ and $y^\top \bq \geq \zeroq$,  there holds $y^\top \Upsilon\hat{\boldsymbol{\Gamma}} = y^\top \bq$.
	So there is 

	\begin{equation}\label{dual}
		|y^\top \bq| =  |y^\top \Upsilon\hat{\boldsymbol{\Gamma}}|  \leq  |C^\top\hat{\boldsymbol{\Gamma}}|.
	\end{equation}
	The inequality of  \eqref{dual} holds because $\hat{\boldsymbol{\Gamma}}\ge 0$ and both $\Upsilon^\top y$ and $C$ are real vectors. Then the inequality \eqref{e:w-dual} follows.
\end{proof}

Next, we will explore the strong dual expression of the QLP problem \eqref{e:w-dual}. Before we can do that, we need to derive several new properties of convex quaternion sets. 

\subsection{Separation theorem of convex quaternion vector sets}
At first, we introduce the definition and properties of convex quaternion vector sets.
\begin{definition}
	\label{def:qconvex} 
	The quaternion vector set  $\Sb \subset \mathbb{Q}^n$ is convex  if  for any two quaternion vectors $\qq_1,~\qq_2 \in \Sb$, $\qq_1\alpha  + \qq_2(1-\alpha) \in \Sb$, where $\alpha \in [0,1]$. 
\end{definition}

\begin{lemma}\label{lem:intersetsum}
	Let $\Sb_1$ and $\Sb_2$ be two convex quaternion vector sets. Then
	\begin{enumerate}[(a)]
		\item $\Sb_1 \cap \Sb_2$ is a convex quaternion vector set;
		\item and $\Sb_1 \pm \Sb_2 = \{\qq_1 \pm \qq_2~|~q_1 \in \Sb_1, \qq_2 \in \Sb_2\}$ is a convex quaternion vector set. 
	\end{enumerate}	 
\end{lemma}

\begin{proof}
	Two assertions can be proved by the straightforward calculation. 
\end{proof}

Now we present the projection theorem of convex quaternion vector sets.
\begin{theorem}\label{t:unique}
	Let $\Sb \subset \mathbb{Q}^{n}$ be a nonempty convex quaternion vector set and quaternion vector $\yq \in \mathbb{Q}^{n}\setminus\Sb$. Then there exists a unique point $\hat{\xq} \in \Sb$ that has the minimum distance from $\yq$. 
\end{theorem}

\begin{proof}
	What we need to prove is that there exists  a unique solution to the optimal minimization problem
	\begin{equation}\label{e:omp}
		\|\hat{\xq}-\yq\|=\inf\limits_{\xq \in \Sq}\big\{\| \xq - \yq \|\big\}.
	\end{equation}
	Let $\Bq = \{\zq \mid \; \| \zq \|  \leq 1 ,\zq \in \mathbb{Q}^{n}   \}$. Then we have a closed set $\Dq = \Sb \cap (\{\yq\} + \beta \Bq )$, where $\beta$ is a sufficiently large positive real number. Obviously, there exists a point $\hat{\xq} \in \Dq \subset \Sb$ such that the continuous function $f(\xq)= \; \| \xq - \yq \| $ reaches the  minimum. Denote $r=f(\hat{\xq})$. 
	
	Suppose there exists another point $\tilde{\xq} \in \Sq $ such that $f(\tilde{\xq})= r$.
	Let $\dot{\xq} = \frac{1}{2}(\hat{\xq} + \tilde{\xq})$, that is,  a convex combination of $\hat{\xq}$ and  $\bar{\xq}$.
	Since $\hat{\xq},~\tilde{\xq}\in\Sq$ and $\Sq$ is convex, there is $\dot{\xq}\in\Sq$. Thus, we have   $f(\dot{\xq})\ge f(\hat{\xq}) =f(\tilde{\xq})= r$.
	However,  
	$$\| \dot{\xq} - \yq \| =\frac{1}{2}\left\| \hat{\xq} - \yq + \tilde{\xq} - \yq\right\|   
	\leq \frac{1}{2}\| \hat{\xq} - \yq \| + \frac{1}{2}\|\tilde{\xq} - \yq \| = r.
	$$
	That implies $\| \dot{\xq} - \yq \| \; = r$. According to the parallelogram rule, 
	$$
	{\| \hat{\xq} - \tilde{\xq} \|}^2 = 
	2{\| \hat{\xq} - \yq \|}^2 + 2{\| \tilde{\xq} - \yq \|}^2 - 4{\| \dot{\xq} - \yq \|}^2 = 0. 
	$$
	The uniqueness is proven, i.e., $\hat{\xq} = \tilde{\xq}$.
\end{proof}

Based on the above projection theorem, we can present the separation theorem of  convex quaternion vector sets and the quaternion Farkas lemma. Let $\mathbb{Q}_+^n$ and  $\mathbb{Q}_-^n$   denote the sets of all nonnegative and nonpositive  $n$-dimensional quaternion vectors, respectively.
\begin{definition}\label{def:hyperplane }
	Suppose  $\Sq_1$ and $\Sq_2$ are two convex quaternion vector sets. If there exists a real vector  $p \in \mathbb{R}^{n}$ and a quaternion $\alphaq \in \mathbb{Q}$ such that  $\Sq_1\subset \{\xq\in \mathbb{Q}^{n}~|~p^{T}\xq\leq\alphaq\}$ and $\Sq_2\subset \{\xq\in \mathbb{Q}^{n}~|~p^{T}\xq\ge \alphaq\}$, then the quaternion hyperplane $H=\{\xq\in \mathbb{Q}^{n}~|~p^{T}\xq = \alphaq\}$ separates $\Sq_1$ and $\Sq_2$.
\end{definition}

\begin{theorem} \label{th:css}
	Let $\Sq \subset \mathbb{Q}^{n}$ be a closed convex set including $\zeroq$. For each quaternion vector $\yq \in \mathbb{Q}_{+}^n\! \cup\! \mathbb{Q}_{-}^n$ and $\yq\notin\Sq$, there exists a real vector $p \in \mathbb{R}^{n}$ and a quaternion $\alphaq \in \mathbb{Q}$ such that $p^{T}\xq \leq \alphaq < p^{T}\yq$ holds for every $\xq \in \Sq$.
\end{theorem}
\begin{proof}
	According to Theorem \ref{t:unique}, there exits a unique $\hat{\xq} \in \Sq$ such that \eqref{e:omp} holds. Such $\hat{\xq}$ must belong to $\mathbb{Q}_{+}^n \cup\mathbb{Q}_{-}^n$   since $\Sq$ is a closed convex set including $\mathbf{0}$ and $\yq$ is a nonnegative or nonpositive quaternion vector that does not belong to $\Sq$.
	Without loss of generality, we assume that $\yq,~\hat{\xq}\in\mathbb{Q}_+^n$ and $\Sq = [0,s_0]^n+[0,s_1]^n\iq+ [0,s_2]^n\jq+[0,s_3]^n\kq$ $:=\{\xq=x_0+x_1\iq+x_2\jq+x_3\kq~|~x_j\in [0,s_j]^n$\}, where $s_j\ge 0$ and $[0,s_j]^n:=\{x=[x_i]_{i=1}^n\in\mathbb{R}^n~|~x_i\in [0,s_j]\}$, $j=0,\ldots,3$. 
	
	Denote $\yq=[\yq_i]_{i=1}^n$ and $\hat{\xq}=[\hat{\xq}_i]_{i=1}^n$. Since $\yq\notin\Sq$, there  exists one element of $\yq$, say $\yq_i$, which does not belong to the quaternion set   $[0,s_0]+[0,s_1]\iq+ [0,s_2]\jq+[0,s_3]\kq$. Let $\yq_i = y_{i0} + y_{i1}\iq + y_{i2}\jq + y_{i3}\kq$. Then there at least one holds that $y_{i0}> s_0,~y_{i1}>s_1,~y_{i2}>s_2$ or $y_{i3}>s_3$. For simplicity, we assume that $y_{i0}>s_0,~y_{i1}\in [0,s_1],~y_{i2}\in [0,s_2],~ y_{i3}\in [0,s_3]$. Then  $
	\hat{\xq}_i = s_0 + y_{i1}\iq + y_{i2}\jq + y_{i3}\kq$. Set $\alphaq = \frac{s_0 + y_{i0}}{2} +y_{i1}\iq + y_{i2}\jq + y_{i3}\kq$.
	Then, with choosing $p=[p_k]_{k=1}^n\in\mathbb{R}^n$ with $p_i=1$ and $p_k=0$ if $k\ne i$.
	Then we have  $p^{T}\hat{\xq} \leq \alphaq < p^{T}\yq$.  On the other hand, for every $\xq\in\Sq$ there holds $p^T\xq\le p^T\hat{\xq}$. 
	So we have proved the assertion  $p^{T}\xq \leq \alphaq < p^{T}\yq$ holds for every $\xq\in\Sq$.
	The other cases can be proved in the same way.
\end{proof}

\begin{theorem} (Quaternion Farkas lemma)
	\label{th_Farkas}
	Suppose  $ \Upsilon \in\mathbb{R}^{n \times m}$ is a real matrix with $\Upsilon_{ij} \Upsilon_{ik}>0$, $ j,~k\in\{1,2,...,n\}$, and $\mathbf{b}\in\mathbb{Q}^m$ is a nonnegative quaternion vector. There  exactly holds one of the following assertions:
	\begin{enumerate}[(a)]
		\item \label{item-a} There exists $\boldsymbol{\Gamma}\in\mathbb{Q}_+^{n}$ such that $\Upsilon \boldsymbol{\Gamma}=\bq$;
		\item \label{item-b} There exists ${y}\in\mathbb{R}^{n}$ such that ${\Upsilon}^{T}{y}\leq  0$ and ${\bq}^{T}{y}>\zeroq.$	
	\end{enumerate}	 
\end{theorem}

\begin{proof} Firstly, we assume that \eqref{item-a} holds.  That means there exists a quaternion vector $\hat{\boldsymbol{\Gamma}} \geq 0$ such that $\Upsilon \hat{\boldsymbol{\Gamma}}=\bq$. For any  $y\in\mathbb{R}^n$ satisfying ${\Upsilon}^{T} y \leq 0$, we have $\bq^{T} y = \hat{\boldsymbol{\Gamma}}^{T} {\Upsilon}^{T} y \leq \zeroq$. This is clearly inconsistent with ${\bq}^{T}y>0$ under the above assumption. Then we have proved that \eqref{item-b} does not hold.
	
	Secondly, we assume that \eqref{item-a} does not hold. Denote $\Sq = \{\zq \mid \zq = \Upsilon \boldsymbol{\Gamma}, \boldsymbol{\Gamma} \geq \zeroq\}$. Clearly,  $\Sq$ is a nonempty convex quaternion vector set and $\bq \notin \Sq$. According to Theorem \ref{th:css}, there exists a real vector $y \in \mathbb{R}^{m}$ and a quaternion $\alphaq \in \mathbb{Q}$ such that $y^{T}\zq \leq \alphaq < y^{T}\bq$ holds for every $\zq\in\Sq$. It is clear that $\alphaq \geq \zeroq$ because $\zq=\zeroq \in \Sq$. So $\bq^{T}y > \zeroq$ and $\alphaq \geq y^{T}\zq = y^{T} {\Upsilon} \boldsymbol{\Gamma} = \boldsymbol{\Gamma}^{T} {\Upsilon}^{T}y$.  Since $\alphaq$ is bounded by $y^T\bq$, the inequality  $\boldsymbol{\Gamma}^{T} {\Upsilon}^{T}y\le \alphaq$ holds for every nonnegative quaternion vector $\boldsymbol{\Gamma}$ if and only if ${\Upsilon}^{T}y\le 0$. That is,  \eqref{item-b} holds.
	
	With all above results in hand, we have completed the proof.  
\end{proof}

\subsection{Quaternion strong dual theorem}

Based on Theorem \ref{th_Farkas}, we present the strong duality form of the QLP problem \eqref{e:lpp}. That is, the inequality in \eqref{e:w-dual} is improved into the equality. 

\begin{theorem}
	\label{th_dual2}
	Given  quaternion vector $\boldsymbol{b} \in \mathbb{Q}^{n}$, nonnegative real vector $C\in \mathbb{R}_+^{m}$ and real matrix $ \Upsilon\in \mathbb{R}^{n \times m}$, there is
	\begin{equation}\label{e:s-dual}
	\max_{y\in \mathbb{R}^{n}}\left\{|\bq^\top y|\mid \Upsilon^\top y \leq C, \bq^\top y \geq \zeroq\right\} = \min_{\boldsymbol{\Gamma} \in \mathbb{Q}^{m}}\left\{|C^\top\boldsymbol{\Gamma}|\mid \Upsilon\boldsymbol{\Gamma}=\bq,\boldsymbol{\Gamma}\geq \zeroq\right\}.
	\end{equation}
\end{theorem}

\begin{proof}
	Suppose the minimum value of the right hand of \eqref{e:s-dual} is achieved at $\hat{\boldsymbol{\Gamma}}$ and denote $\hat{\zq} := C^{T} \hat{\boldsymbol{\Gamma}}$. 
	As $C\in\mathbb{R}_+^m$ and $\hat{\boldsymbol{\Gamma}}\ge \zeroq$, $\hat{\zq}\ge\zeroq$.
	Define
	$$ \hat{\Upsilon}:=\begin{bmatrix}\Upsilon\\-C^{T}\end{bmatrix}~\text{and}~ \hat{\bq}_{\boldsymbol{\epsilon}}:=\begin{bmatrix}\bq\\ -\hat{\zq}+\boldsymbol{\epsilon}\end{bmatrix},$$
	where $\boldsymbol{\epsilon}\in\mathbb{Q}$ is a positive quaternion.
	When $|\hat{\zq}-\boldsymbol{\epsilon}|<|\hat{\zq}|$, 
	there is no  $\boldsymbol{\Gamma}\geq \zeroq$ satisfying both $ \Upsilon\boldsymbol{\Gamma}=\bq$ and $C^T\boldsymbol{\Gamma}=\hat{\zq}-\boldsymbol{\epsilon}$.
	It implies that  $\hat{\Upsilon} \boldsymbol{\Gamma} \neq \hat{\bq}_{\boldsymbol{\epsilon}}$ for any $\boldsymbol{\Gamma} \geq 0$. 
	This means the assertion \eqref{item-a} of Lemma \ref{th_Farkas}  does not hold. So the assertion \eqref{item-b}  holds according to such quaternion Farkas lemma. 
	That is, there exists a  real vector  $\left.\hat{y}=\left[\begin{matrix}y_0\\ \beta\end{matrix}\right.\right ] $ such that $\hat{\Upsilon}^{T} \hat{y} \leq 0$ and $\hat{\bq}_{\boldsymbol{\epsilon}}^{T}\hat{y}>\zeroq$. This is equivalent to $\Upsilon^{T}y_0 \leq \beta C$ and $\bq^{T} y_0 >\beta(\hat{\zq} - \boldsymbol{\epsilon})$.

	Next, we need to prove $\beta > 0$.
	Firstly, $\hat{\bq}_{\boldsymbol{\epsilon}}$ is limited by $\hat{\bq}:=\left[\begin{matrix}\bq\\ -\hat{\zq}\end{matrix}\right]$ as $\boldsymbol{\epsilon}$ tends to $\zeroq$.
	At the limit, one can get   $\hat{\Upsilon} \boldsymbol{\Gamma} = \hat{\bq}$. 
	Since  $\hat{\Upsilon}^{T} \hat{y}\leq 0$, we have $\hat{\bq}^{T}\hat{y}= \boldsymbol{\Gamma}^T\hat{\Upsilon}^T\hat{y}\leq 0$. 
	However, we can know $0 <\hat{\bq}_{\boldsymbol{\epsilon}}^{T}\hat{y} = \hat{\bq}^{T}\hat{y} +\beta \boldsymbol{\epsilon}$. That implies  $\beta>0$.

	By dividing such $\beta$ from  both sides of $\hat{\Upsilon}^{T} \hat{y} \leq 0$ and $\hat{\bq}_{\boldsymbol{\epsilon}}^{T}\hat{y}>\zeroq$, we have  $\Upsilon^{T}\frac{y_0}{\beta}  \leq  C, \bq^{T} \frac{y_0}{\beta} >(\hat{\zq} - \boldsymbol{\epsilon})$. 
	So $\max_{y}\left\{|\bq^\top y|\mid \Upsilon^\top y\leq\ C, \bq^\top y \geq \zeroq \right\} >|(\hat{\zq} - \boldsymbol{\epsilon})|>|\hat{\zq}| -| \epsilon|$. 
	Because $\max_{\yq}\left\{|\bq^\top y|\mid|\Upsilon^\top y|\leq C, \bq^\top y \geq \zeroq\right\} \leq |\hat{\zq}|$, 
	there is  $\max_{y}\left\{|\bq^\top y|\mid \Upsilon^\top y\leq C , \bq^\top y \geq0\right\} = |\hat{\zq}| = \min_{\boldsymbol{\Gamma}}\left\{|C^\top\boldsymbol{\Gamma}|\mid \Upsilon\boldsymbol{\Gamma}=\bq,\boldsymbol{\Gamma}\geq \zeroq\right\} , \boldsymbol{\Gamma} \in \mathbb{Q}^{m} ,y\in \mathbb{R}^{n}$.
\end{proof}

A form of strong duality has been developed for the  QLP problem. It will be applied to  WQGAN.

\section{Dual form of quaternion Wasserstein distance with application to Wasserstein quaternion  generative model} \label{sec:wqgan}

In this section,  we present the dual form of QWD and use it to construct a new color image generation network, called WQGAN. 

Recalling from \eqref{defw} and \eqref{cons2},   the QWD between  $P_r$ and $P_g$   is defined by
\begin{subequations}
	\label{e:argJS}
	\begin{align}
		\label{e:argJS-1}
		&\mathcal{W}[P_r,P_g]=\inf_{\gamma\in \Pi(P_r, P_g)} \left\{
		\sum_{\xq\in S_r}\sum_{\yq \in S_g} c(\xq,\yq)\gamma(\xq,\yq)\right\}\\
		\label{e:argJS-2} 
		&\textrm{s.t. }~\sum_{\yq\in S_g}\gamma(\xq,\yq)=\pq_r(\xq),\\
		&\quad ~~\sum_{\xq\in S_r} \gamma(\xq,\yq)=\pq_g(\yq),
	\end{align}
\end{subequations}
where $c(\xq,\yq)$ denotes the distance  between quaternion vectors $\xq$ and $\yq$, for instance, $c(\xq,\yq)=\|\xq-\yq\|$.
This novel QWD can be applied to the optimal transportation between these two distributions of quaternion vectors, as shown in Figure \ref{p:earth_move}.

\begin{figure}[ht]
	\centering
	\setlength{\abovecaptionskip}{-0cm}
	\setlength{\belowcaptionskip}{-0cm}
	\vspace{-0.1cm}
	\includegraphics[width=0.5\textwidth,height=0.15\textwidth]{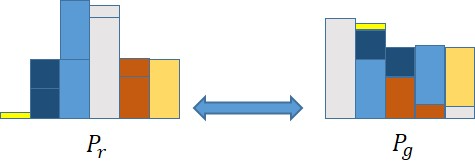}
	\caption{
		The optimal transportation between  $P_r$ and $P_g$.
	}
	\label{p:earth_move}
	
\end{figure}

\subsection{Dual form of quaternion  Wasserstein distance}

Firstly, we provide the  dual form of QWD. 
The discrete forms of   $\gamma(\xq,\yq)$ and $c(\xq,\yq)$ are
\begin{subequations}\label{e:gamma-C}
	\begin{align}
		&\left.\Gamma=\left[\gamma(\xq_1,\yq_1),\gamma(\xq_1,\yq_2),\cdots,\gamma(\xq_2,\yq_1),\gamma(\xq_2,\yq_2),\cdots,\gamma(\xq_n,\yq_1),\gamma(\xq_n,\yq_2),\cdots\right.\right]^T,\\
		&\left.C=\left[c(\xq_1,\yq_1),c(\xq_1,\yq_2),\cdots,c(\xq_2,\yq_1),c(\xq_2,\yq_2),\cdots,c(\xq_n,\yq_1),c(\xq_n,\yq_2),\cdots\right.\right]^T,
	\end{align}
\end{subequations}
where $\xq_1,\xq_2,\ldots$ and $\yq_1,\yq_2,\ldots$ constitute the supports $S_r$ and $S_g$, respectively. 
Then we transform the constraint $\gamma \in \Pi (P_r,P_g)$ into  the following discrete form
\begin{equation}\label{e:UpsilonGammab}
	\Upsilon\Gamma=\bq,
\end{equation}
where 
\begin{equation}\label{e:Upsilon}
	\left.{\Upsilon}=\left(\begin{array}{ccc|ccc|c|ccc|c}	1&1&\cdots&0&0&\cdots&\cdots&0&0&\cdots&\cdots\\0&0&\cdots&1&1&\cdots&\cdots&0&0&\cdots&\cdots\\	\vdots&\vdots&\ddots&\vdots&\vdots&\ddots&\ddots&\vdots&\vdots&\ddots&\ddots\\0&0&\cdots&0&0&\cdots&\cdots&1&1&\cdots&\cdots\\\vdots&\vdots&\ddots&\vdots&\vdots&\ddots&\ddots&\vdots&\ddots&\ddots&\ddots\\\hline1&0&\cdots&1&0&\cdots&\cdots&1&0&\cdots&\cdots\\0&1&\cdots&0&1&\cdots&\cdots&0&1&\cdots&\cdots\\\vdots&\vdots&\ddots&\vdots&\vdots&\ddots&\ddots&\vdots&\vdots&\ddots&\ddots\\0&0&\cdots&0&0&\cdots&\cdots&0&0&\cdots&\cdots\\\vdots&\vdots&\ddots&\vdots&\ddots&\ddots&\vdots&\vdots&\ddots&\ddots&\ddots\end{array}\right.\right)
\end{equation} 
is an indicator matrix and 
\begin{equation}\label{e:bqdis}
  \bq=\left(\pq_r(\xq_1),\pq_r(\xq_2),\cdots,\pq_r(\xq_n),\cdots,\pq_g(\yq_1),\pq_g(\yq_2),\cdots,\pq_g(\yq_n),\cdots\right)^T
\end{equation}
is the splice of long vectors of $\pq_r(\xq)$ and $\pq_g(\xq)$.
So the discrete QWD is described as
\begin{equation*}
	\min_{\boldsymbol{\Gamma}}\left\{\langle\boldsymbol{\Gamma},C\rangle\mid\Upsilon \boldsymbol{\Gamma}=\boldsymbol{b},\boldsymbol{\Gamma}\geq0\right\},
\end{equation*}
where $\langle\boldsymbol{\Gamma},C\rangle = \left| C^T \boldsymbol{\Gamma} \right|, \boldsymbol{\Gamma} \in \mathbb{Q}_+^{\infty}, C\in \mathbb{R}_+^{\infty}$.

From the quaternion strong dual theorem (Theorem \ref{th_dual2}), we get
\begin{equation}\label{e:dualformQ}
	\max_{F} \left\{\langle\boldsymbol{b},F \rangle\mid\Upsilon^\top F\leq C,   \langle\boldsymbol{b},F \rangle \geq 0 \right\}
	= 
	\min_{\boldsymbol{\Gamma}}\left\{\langle\boldsymbol{\Gamma},C\rangle\mid\ \Upsilon\boldsymbol{\Gamma}=\boldsymbol{b},\boldsymbol{\Gamma}\geq0\right\}.
\end{equation}
Note that $\boldsymbol{b}$ is the splice of two parts.  Similarly,  we can also write $F$ as
\begin{equation*}
	\left.F=\left[f(\xq_1),f(\xq_2),\cdots,f(\xq_n),\cdots,g(\yq_1),g(\yq_2),\cdots,g(\yq_n),\cdots\right.\right]^{T},
\end{equation*}
where $f$ and $g$ are two real-valued functions. 
Then we  obtain
\begin{equation}\label{e:bFinner}
	\langle\boldsymbol{b},F \rangle=\left| \sum_n\pq(\xq_n)f(\xq_n)+\sum_n\qq(\yq_n)g(\yq_n)\right|
\end{equation}
and
the condition $\Upsilon^\top F\leq C$ becomes
\begin{equation}\label{e:conscond}
	\left.\left(
	\begin{array}{ccccc|ccccc}
		1&0&\cdots&0&\cdots&{1}&0&\cdots&0&\cdots\\
		1&0&\cdots&0&\cdots&{0}&1&\cdots&0&\cdots\\
		{\vdots}&{\vdots}&\ddots&{\vdots}&\ddots&{\vdots}&{\ddots}&{\vdots}&\ddots\\
		\hline{0}&1&\cdots&0&\cdots&{1}&0&\cdots&0&\cdots\\
		{0}&1&\cdots&0&\cdots&{0}&1&\cdots&0&\cdots\\
		\vdots&\vdots&\ddots&\vdots&\ddots&\vdots&\vdots&\ddots&\vdots&\ddots\\
		\hline{\vdots}&{\vdots}&\ddots&\vdots&\ddots&{\vdots}&\vdots&\ddots&\vdots&\ddots\\\hline{0}&{0}&\cdots&{1}&\cdots&{1}&{0}&\cdots&{0}&\cdots\\{0}&{0}&\ddots&{1}&\ddots&{0}&1&\ddots&{0}&\ddots\\{\vdots}&\vdots&\ddots&\vdots&\ddots&{\vdots}&\vdots&\ddots&\vdots&\ddots\\\hline{\vdots}&\vdots&\ddots&\vdots&\ddots&{\vdots}&\vdots&\ddots&\vdots&\ddots
	\end{array}\right.\right)
	\left.\left(\begin{array}{c}f(\xq_1)\\f(\xq_2)\\\vdots\\f(\xq_n)\\\vdots\\g(\yq_1)\\g(\yq_2)\\\vdots\\g(\yq_n)\\\vdots\end{array}\right.\right)
	\leq
	\left.\left(\begin{array}{c}c(x_1,y_1)\\c(x_1,y_2)\\\vdots\\\hline c(x_2,y_1)\\c(x_2,y_2)\\\vdots\\\hline\vdots\\\hline c(x_n,y_1)\\c(x_n,y_2)\\\vdots\\\hline\vdots\end{array}\right.\right).
\end{equation}

Rewriting \eqref{e:bFinner} and \eqref{e:conscond} into  the forms of quaternion random variables, 	we have the dual form of the QWD  \eqref{e:argJS} as

\begin{equation}\label{e:dualQWD}
\begin{split}
\mathcal{W}[P_r,P_g]
&= \max_{\substack{f,g:\mathbb{Q}^n \to \mathbb{R}}} \LARGE\{ \sum_{\xq \in S_r \cup S_g} \left[ \pq(\xq)f(\xq) + \qq(\xq)g(\xq) \right] \ |  \\
& \quad  f(\xq) + g(\yq) \leq \|\xq - \yq\|,\ \forall \xq, \yq \in \mathbb{Q}^n \LARGE\}.
\end{split}
\end{equation}

\subsection{Wasserstein Quaternion Generation Adversarial Neural Network}	

By defining $g(\xq):=-f(\xq)$ in \eqref{e:dualQWD}, we obtain the dual form of the optimal transportation cost \eqref{e:argJS}. Specifically, in the following sections, we extend $||\xq-\yq||$ as $c(\xq,\yq)$,

\begin{equation}\label{e:qwdg=-f}
\begin{split}
\mathcal{W}[P_r,P_g]
&= \max_{f: \mathbb{Q}^n \to \mathbb{R}} \LARGE\{ \sum_{\xq \in S_r \cup S_g} \left[ \pq(\xq)f(\xq) - \qq(\xq)f(\xq) \right] \ |  \\
& \quad  f(\xq) - f(\yq) \leq c(\xq,\yq),\ \forall \xq, \yq \in \mathbb{Q}^n \LARGE\}.
\end{split}
\end{equation}

This form is what we are ultimately looking for.
Since $P_r,P_g$ are both probability distributions, we can write equation \eqref{e:qwdg=-f} in sampling form:
\begin{equation}\label{e:final_loss}
	\mathcal{W}[P_r,P_g]=\max_{\|f\|_{\textrm{Lip}}\leq1}
	\left\{\mathbb{E}_{\xq\sim P_r}[f(\xq)]-\mathbb{E}_{\xq\sim P_g}[f(\xq)] \right\}.
\end{equation}
Here, $\|f\|_{\textrm{Lip}}$ stands for the Lipschitz norm of $f$. 
If the QWD is reduced to the  Wasserstein distance, then 
equation \eqref{e:final_loss} is consistent with Theorem $5.10$ in \cite{r16chen}.

In practice, it is  unreasonable to use  Equation \eqref{defw} to measure the images of two databases.   In Equation \eqref{e:final_loss},  we use a functional function  that satisfies the Lipschitz condition to analyze the images and apply QWD to simplify the calculation process. Compared with the Wasserstein distance of the real number, QWD is directly applied to the  color images for analysis, which ensures the integrity of the model theory.

Based on equation \eqref{e:final_loss}, we utilize the dual form of QWD to guide the updates of the discriminators and generators of QGANs  \cite{r2,r11,r13}.  The generator is constructed by quaternion modules and both the generated  and real color images are represented by quaternion matrices. The weights of the discriminator are within a certain interval to ensure the satisfaction of Lipschitz condition ($\|f\|_L\leq1$). The most important thing is that the difference between the discriminative results of the generated color images and the real color images from database is computed by the value of the loss function, which  represents the QWD between them.

The specific training method is given in Algorithm \ref{alg:C}.
At the beginning of training, the model parameters 
are initialized randomly. 
The discriminator parameters are optimized through the inner loop (lines 3-9).   
The discriminator optimization method uses the dual form of QWD, as shown in the formula \eqref{e:final_loss}. 
In this operation, the images generated by the quaternion generator along with the real images extracted from the training database form a quaternion training dataset for the discriminator. The difference between the correct discrimination rate of the real images and the correct discrimination rate of the generated images constitutes the loss value of the discriminator. The Root Mean Square Propagation optimization algorithm (RMSProp) \cite{tihi12} is used to update the discriminator parameters. 
In order to ensure that the discriminator parameters satisfy the Lipschitz condition, we also  constrain the updated discriminator parameters into a specified range, limiting the extreme values of the parameters. The parameters that exceed the specified range will be reassigned. 
Similarly, the quaternion generator parameters are updated by the outer loop (lines 2-13).
In each iteration, new quaternion noise vectors are randomly generated and provided to the quaternion generator to generate color images. The generated images are then discriminated by the updated discriminator. The correct discrimination rate is the loss value of the quaternion generator, and RMSProp is used to optimize the parameters.

\begin{algorithm}
	\caption{Wasserstein Quaternion Generation Adversarial Neural Network Algorithm}
	\label{alg:C}
	\begin{algorithmic}[1] 
		\STATE  Initialize discriminator parameter $d$, generator parameter $\theta$, number of iterations $Iter$, number of samples per batch $m$, number of discriminator parameter updates per iteration $k$, learning rate $\sigma$,  and  clipping parameter $c$.
		\FOR{$i=1,2,\cdots,Iter$}
		\FOR{$j=1,2,\cdots,k$}
		\STATE Select $m$ noise sample points: $\zq_{1}, \zq_{2},..., \zq_{m}$;
		\STATE  Select $m$  real color images: $\xq_{1}, \xq_{2},..., \xq_{m}$;
		\STATE  $loss_{d}=\nabla_{d}[\frac{1}{m}\sum_{s=1}^{m}f_{d}(\xq_{\mathrm{s}})-\frac{1}{m}\sum_{s=1}^{m}f_{d}(\gq_{\theta}(\mathrm{\zq}_{\mathrm{s}}))]$;
		\STATE $d=d+\sigma {\rm RMSProp}(w, loss_{d})$;
		\STATE $d=d+\text{limit}(w,-c,c)$;
		\ENDFOR
		\STATE  Select m sample points from a known noise distribution $p_{\zq}(\zq):$ $\zq^{1}, \zq^{2},..., \zq^{m}$.
		\STATE $loss_{g}=-\nabla_{d}[\frac{1}{m}\sum_{s=1}^{m}f_{d}(\gq_{\theta}(\zq^{s}))]$;
		\STATE  $\theta=\theta-\sigma {\rm RMSProp}(w,loss_g)$.
		\ENDFOR
	\end{algorithmic}
\end{algorithm}

Compared to the QGAN algorithm in \cite{r13}, we have improved the method of calculating the loss value using the dual form of QWD, removed the quaternion normalization process of the data and replaced the optimization function with {\rm RMSProp} function \cite{tihi12}. 

\section{Experiments }\label{sec:exs}
In this section, we compare the newly proposed WQGAN with  GAN \cite{r5}, WGAN  \cite{r7}, QGAN \cite{r12} and Denoising Diffusion Probabilistic Model (DDPM) \cite{r14} in the color image generation. 
In all experiments of this section, the batch-size is set to 64 and the learning rate is set to 0.0002. All GAN-based models adopt the same network architecture, with differences only in the convolution methods and loss functions. For the DDPM, a four-layer upsampling of the same scale as that used in other models is employed.

\begin{figure}[ht]
	\centering
	\setlength{\abovecaptionskip}{-0cm}
	\setlength{\belowcaptionskip}{-0cm}
	\vspace{-0.1cm}
	\includegraphics[width=0.8\textwidth,height=0.2\textwidth]{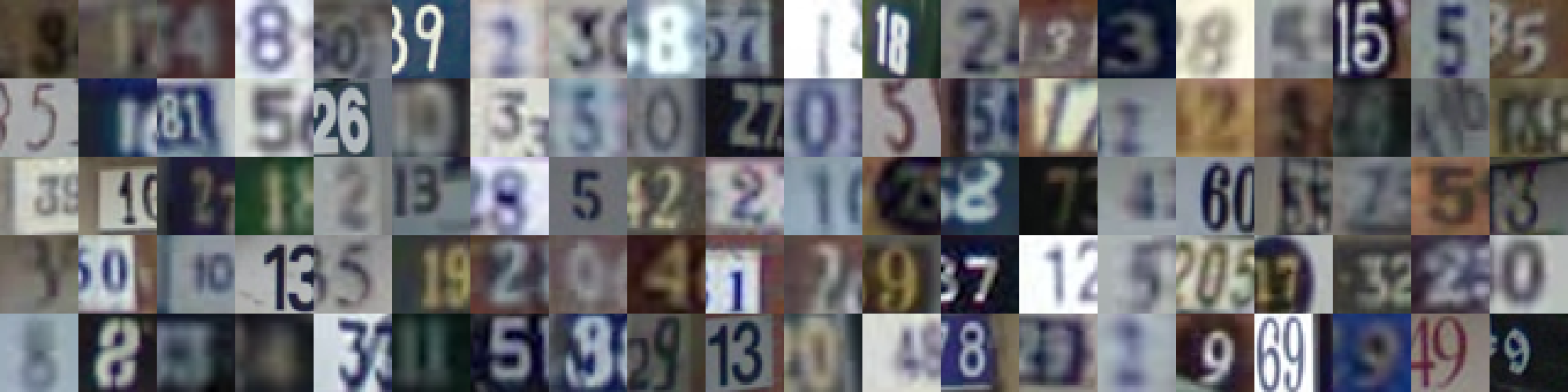}
	\caption{Samples from SVHN database.}
	\label{p:SVHN}
\end{figure}

\begin{figure}[ht]
	\centering
	\setlength{\abovecaptionskip}{-0cm}
	\setlength{\belowcaptionskip}{-0cm}
	\vspace{-0.1cm}
	\includegraphics[width=0.8\textwidth,height=0.2\textwidth]{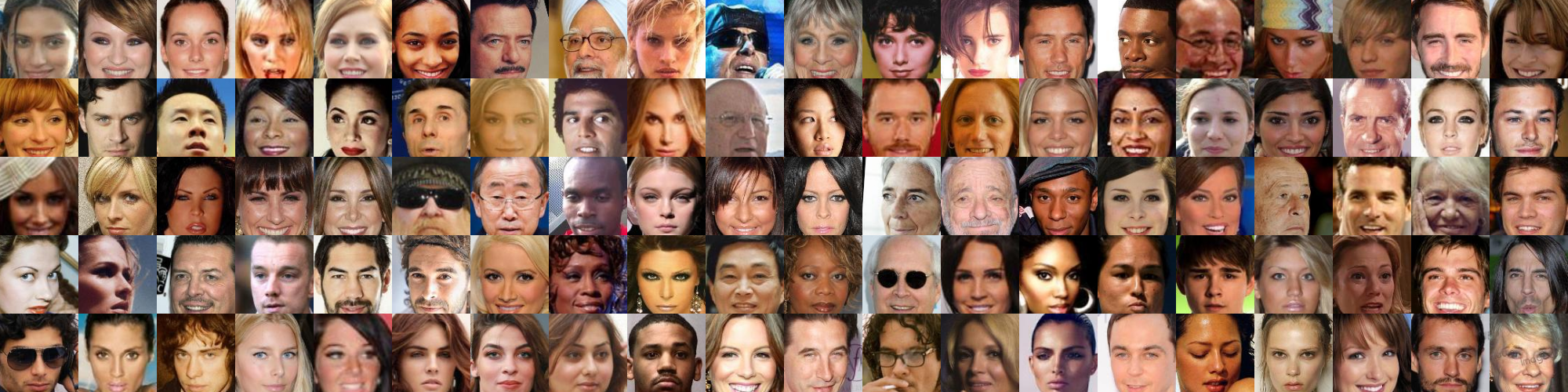}
	\caption{Samples from CelebA database.}
	\label{p:CelebA}
\end{figure}

The Inception Score (IS)\cite{is2016} and Fréchet Inception Distance (FID)\cite{fid2017}are used as the criteria for assessing the quality of model creation.

The Inception Score is a quantitative measure used to analyze the performance of generative models, particularly in terms of the diversity and quality of the images they generate. 
The following formula calculates IS:
\begin{equation}
	IS(\mathrm{P_g})=\exp(\mathbb{E}_{\xq \sim\mathrm{P}_g} KL ( p(\yq|\xq))||p(\yq))),
\end{equation}
where $\xq$ represents a given data point, $\yq$ denotes a predefined label, $p(\yq|\xq)$ is the probability of $\yq$ given $\xq$, $p(\yq)$ is the distribution of labels. $KL$ stands for the Kullback-Leibler divergence.
The FID metric quantifies the resemblance between two sets of images based on the statistical characteristics of their computer vision properties. A FID score of 0 represents the highest level of similarity between two sets of images. The lower the score, the more similar the two sets of images are.

The following formula calculates FID:
\begin{equation}
	FID(P_r,P_g)=||\mu_{P_r}-\mu_{P_g}||_2^2+Tr(\Sigma_{P_r}+\Sigma_{P_g}-2(\Sigma_{P_r}\Sigma_{P_g})^{0.5}),
\end{equation}
where $Tr$ denotes the trace of the matrix, $P_r$ and $P_g$ denote the real images and the generated images, and $\mu, \Sigma$ denotes the mean value and covariance matrix.

Two databases to evaluate algorithms: Street View House Number (SVHN) and CelebA database (CelebA).

The SVHN database consists of 99,289 cropped RGB images of street numbers. The training set comprises 73,257 images, whereas the test set consists of 26,032 images. The numbers in the images are not aligned, and the backgrounds are different.

The CelebA database contains 202,599 face images with various angles and expressions. The training set comprises 200,599 images from the database, while the remaining images form the test set.

\subsection{Experiment}

\begin{table}[ht]
	\centering
	\caption{Generation Images FID and IS with training SVHN database.}
	\label{t:SVHN}

	\small
	\begingroup
	\begin{tabular}{c|ccccc}
		\hline
		\textbf{Iteration} & \textbf{WQGAN} & \textbf{WGAN} & \textbf{QGAN} & \textbf{GAN} & \textbf{DDPM} \\ \hline

		\multicolumn{6}{c}{\textbf{[FID]}} \\ \cline{1-6}
		2000    & \textbf{355.7127}      & 477.8739     	& 533.1914     	& 436.7711    & 914.9643 \\
		4000    & \textbf{372.6435}      & 411.9887      & 501.9213      & 480.9005     & 827.6554 \\
		6000    & \textbf{321.6874}      & 499.7613    	& 451.3789    	& 419.5708    & 722.1165 \\
		8000    & \uline{388.7583}      & \textbf{368.6779}     	& 415.4067     	& 419.1335     & 492.0474 \\
		10000   & \textbf{357.1885}      & 387.0714      & 363.2764      & 367.4091      & 453.3757 \\
		12000   & \uline{354.3042}      & \textbf{344.0127}      & 391.3905      & 439.8063      & 407.3923 \\
		14000   & \uline{360.3284}      & 376.0911      & 368.5414      & 438.5773      & \textbf{345.1264} \\
		16000   & 343.8931      & \uline{335.5764}     	& 337.8915      & 398.5852      & \textbf{327.5645} \\
		18000   & \uline{338.5418}     	& 361.5312    	& 424.0336      & 340.6606      & \textbf{287.9365} \\
		20000   & \textbf{316.0346}     	& 362.4927      & 324.0385     	& 395.4260      & 334.0176 \\ \hline

		\multicolumn{6}{c}{\textbf{[IS ($\pm$ std)]}} \\ \cline{1-6}
		2000    & \textbf{1.38$\pm$0.15} 	& 1.24$\pm$0.10  	& 1.29$\pm$0.09  	& 1.24$\pm$0.09  & \uline{1.25$\pm$0.15} \\
		4000    & \textbf{1.38$\pm$0.25}  	& 1.28$\pm$0.07  	& 1.25$\pm$0.10  	& \uline{1.35$\pm$0.20}  & 1.32$\pm$0.15 \\
		6000    & 1.32$\pm$0.14  	& \uline{1.38$\pm$0.18}     & 1.33$\pm$0.13 	& 1.29$\pm$0.08  & \textbf{1.45$\pm$0.21} \\
		8000    & 1.34$\pm$0.12  	& 1.28$\pm$0.08  	& \textbf{1.62$\pm$0.30}  	& 1.38$\pm$0.10  & \uline{1.50$\pm$0.28} \\
		10000   & 1.30$\pm$0.05  	& \uline{1.51$\pm$0.22}  	& 1.46$\pm$0.17  	& 1.44$\pm$0.14  & \textbf{1.55$\pm$0.30} \\
		12000   & 1.38$\pm$0.10  	& \textbf{1.41$\pm$0.19}  	& 1.29$\pm$0.06  	& \uline{1.39$\pm$0.07}  & 1.31$\pm$0.17 \\
		14000   & 1.39$\pm$0.13  	& 1.34$\pm$0.15  	& \uline{1.40$\pm$0.24}  	& 1.26$\pm$0.10  & \textbf{1.80$\pm$0.59} \\
		16000   & \uline{1.43$\pm$0.23}  	& 1.36$\pm$0.11  	& 1.44$\pm$0.31 	& 1.32$\pm$0.07  & \textbf{1.69$\pm$0.41} \\
		18000   & 1.31$\pm$0.12  	& \uline{1.32$\pm$0.11}  	& 1.25$\pm$0.26 	& \textbf{1.49$\pm$0.25}  & 1.26$\pm$0.04 \\
		20000   & \textbf{1.47$\pm$0.17}  	& \uline{1.41$\pm$0.13}  	& 1.37$\pm$0.19  	& 1.29$\pm$0.10  & 1.36$\pm$0.94 \\ \hline
	\end{tabular}
	\endgroup 
\end{table}

\begin{figure}[ht]
	\centering
	
	\setlength{\abovecaptionskip}{-0cm}
	\setlength{\belowcaptionskip}{-0cm}
	\vspace{-0.1cm}
	\includegraphics[width=0.9\textwidth,height=0.6\textwidth]{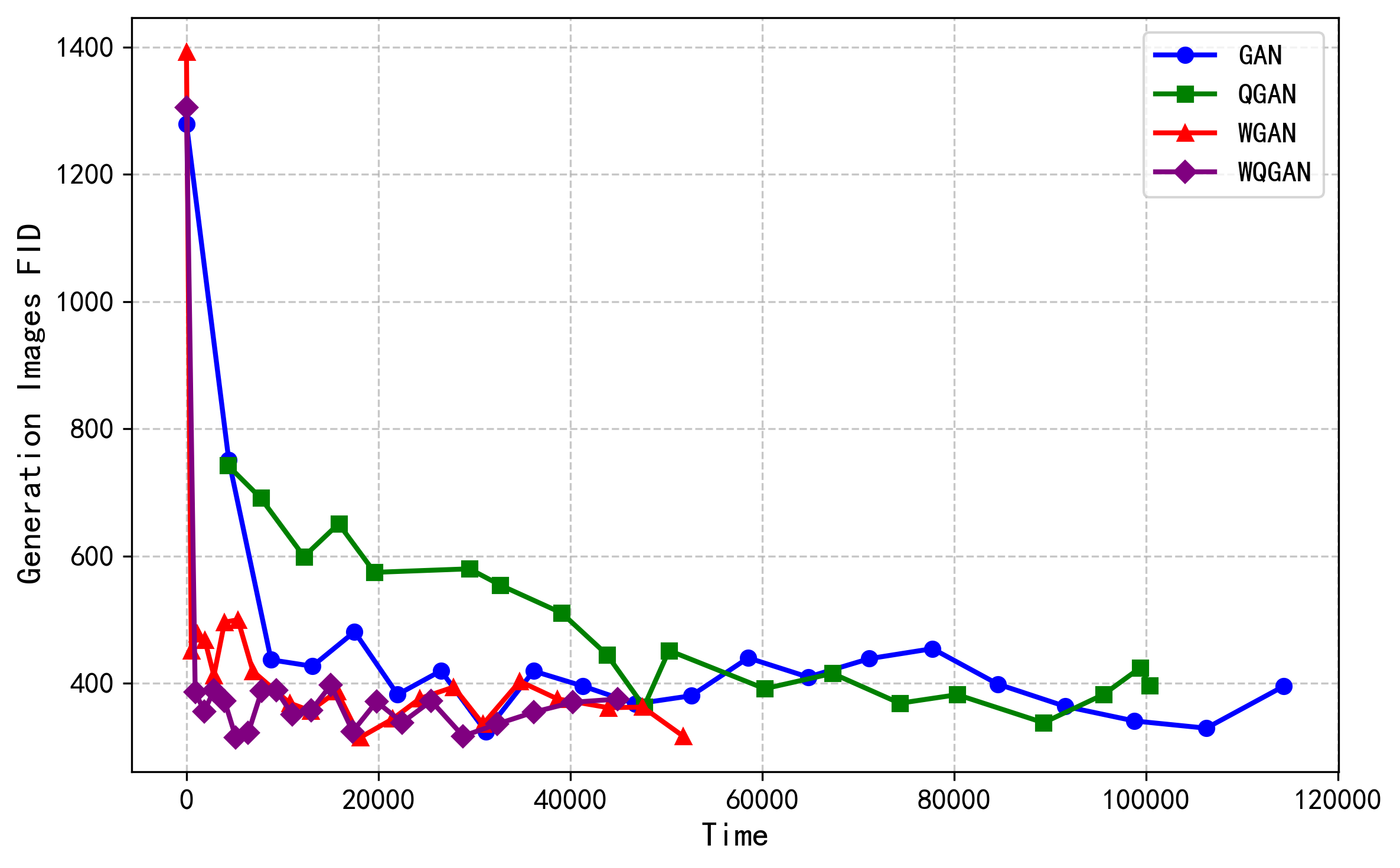}
	\caption{Line graph of FID for generated images (SVHN).}	
	\label{p:svhn_fid_com}
\end{figure}

\begin{figure}[ht]
	\centering
	
	\setlength{\abovecaptionskip}{-0cm}
	\setlength{\belowcaptionskip}{-0cm}
	\vspace{-0.1cm}
	\includegraphics[width=0.9\textwidth,height=0.45\textwidth]{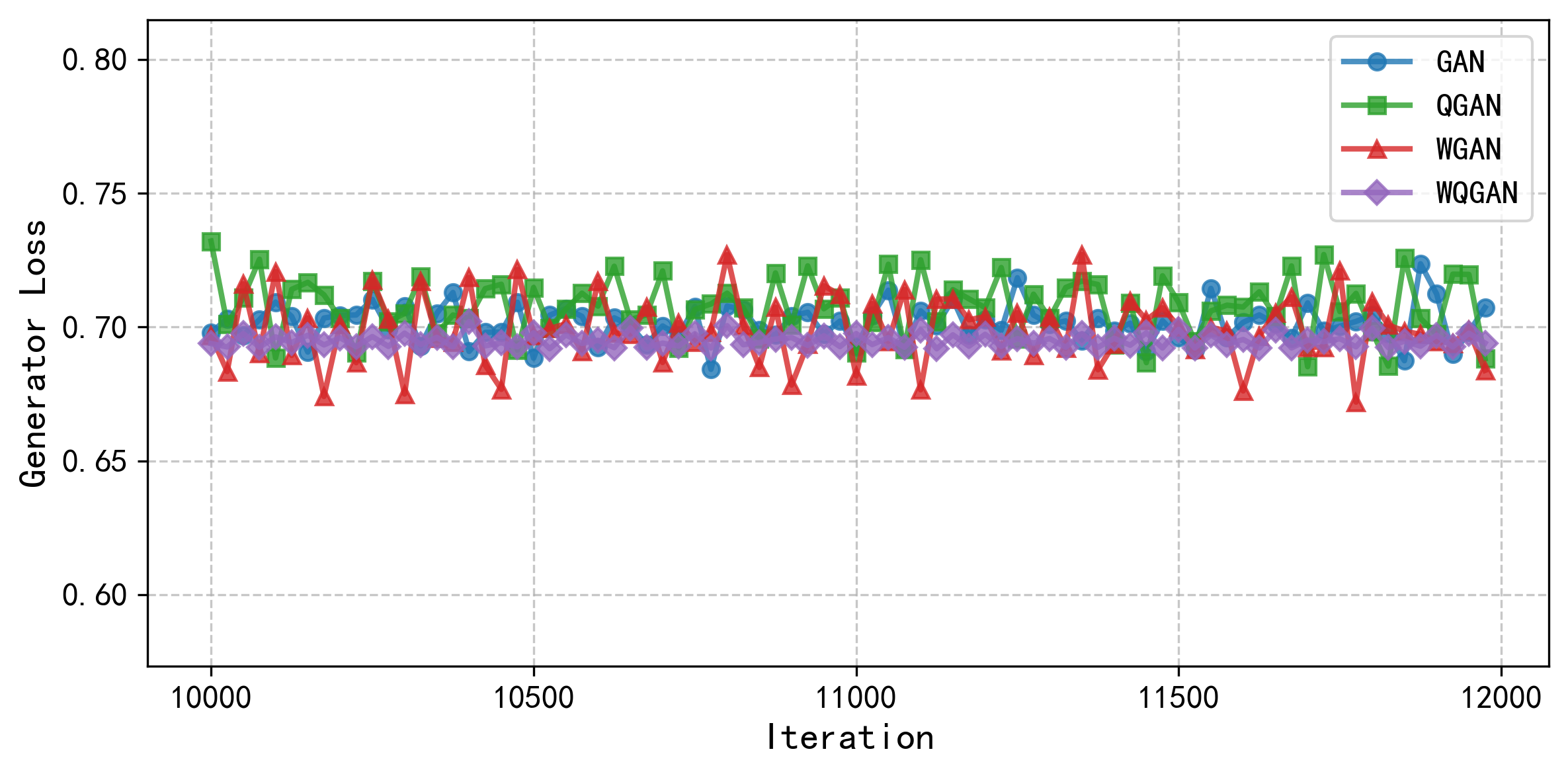}
	\caption{Line graph of generation losses for images (SVHN).}
	\label{p:svhn_loss_com}
\end{figure}

\begin{table}[ht]
	\centering
	\caption{Generation Images FID and IS with training CelebA database.}
    \label{t:Celeba}
    \small
    \begingroup
    \begin{tabular}{c|ccccc}
        \hline
        \textbf{Iteration} & \textbf{WQGAN} & \textbf{WGAN} & \textbf{QGAN} & \textbf{GAN} & \textbf{DDPM} \\ \hline

        \multicolumn{6}{c}{\textbf{[FID]}} \\ \cline{1-6}
		2000    & \textbf{661.8018}      & \uline{711.3851}      & 1311.9301     & 1209.286   & 2298.1365  \\
		4000    & \textbf{643.3912}      & \uline{664.3755}      & 1162.832      & 1134.2018  & 1679.1475    \\
		6000    & \uline{604.7010}      & \textbf{576.4101}      & 1061.5690     & 1160.6651  & 1080.3276   \\
		8000    & \uline{597.2210}      & \textbf{570.8796}      & 931.0366      & 989.9894   & 668.8734   \\
		10000   & \uline{591.8507}      & \textbf{569.2608}      & 989.9965      & 786.8214   & 696.6965   \\
		12000   & \textbf{554.5563}      & \uline{619.2530}      & 910.7812      & 732.9134   & 676.3945    \\
		14000   & \textbf{550.8968}      & 623.5273      & 901.3241      & 715.6228   & \uline{599.8234}   \\
		16000   & \textbf{532.8855}      & \uline{590.8796}      & 931.0366      & 673.3694   & 623.3956   \\
		18000   & \textbf{510.3748}      & 607.1591      & 887.4065      & 645.4864   & \uline{546.1776}   \\
		20000   & \textbf{480.5533}      & 570.2375      & 810.5755      & 610.5495   & \uline{547.6576}   \\ \hline

        \multicolumn{6}{c}{\textbf{[IS ($\pm$ std)]}} \\ \cline{1-6}
        2000    & \uline{1.47$\pm$0.16}  & \textbf{1.56$\pm$0.20}  & 1.33$\pm$0.24  & 1.38$\pm$0.12  & 1.40$\pm$0.28 \\
        4000    & \uline{1.57$\pm$0.23}  & \textbf{1.60$\pm$0.18}  & 1.28$\pm$0.23  & 1.51$\pm$0.28  & 1.43$\pm$0.21 \\
        6000    & \uline{1.56$\pm$0.18}  & 1.52$\pm$0.18  & 1.36$\pm$0.20  & 1.52$\pm$0.17  & \textbf{1.57$\pm$0.38} \\
        8000    & \uline{1.62$\pm$0.20}  & 1.46$\pm$0.18  & 1.21$\pm$0.36  & 1.54$\pm$0.12  & \textbf{1.67$\pm$0.42} \\
        10000   & 1.44$\pm$0.09  & \textbf{1.57$\pm$0.29}  & 1.44$\pm$0.32  & 1.40$\pm$0.16  & \uline{1.56$\pm$0.35} \\
        12000   & 1.55$\pm$0.20  & 1.50$\pm$0.13  & \textbf{1.64$\pm$0.25}  & 1.37$\pm$0.12  & \uline{1.56$\pm$0.44} \\
        14000   & 1.40$\pm$0.12  & 1.46$\pm$0.12  & 1.55$\pm$0.16  & \textbf{1.58$\pm$0.20}  & \uline{1.48$\pm$0.74} \\
        16000   & 1.54$\pm$0.25  & 1.50$\pm$0.29  & 1.46$\pm$0.18  & \textbf{1.64$\pm$0.23}  & \uline{1.64$\pm$0.55} \\
        18000   & 1.52$\pm$0.26  & \uline{1.52$\pm$0.12}  & \textbf{1.57$\pm$0.24}  & 1.42$\pm$0.18  & 1.50$\pm$0.60 \\
        20000   & \textbf{1.66$\pm$0.33}  & \uline{1.62$\pm$0.27}  & 1.54$\pm$0.18  & 1.36$\pm$0.17  & 1.43$\pm$0.21 \\
        \hline
	\end{tabular}
	\endgroup 
\end{table}

\begin{figure}[ht]	
	\centering
	\setlength{\abovecaptionskip}{-0cm}
	\setlength{\belowcaptionskip}{-0cm}
	\vspace{-0.1cm}
	\includegraphics[width=0.9\textwidth,height=0.6\textwidth]{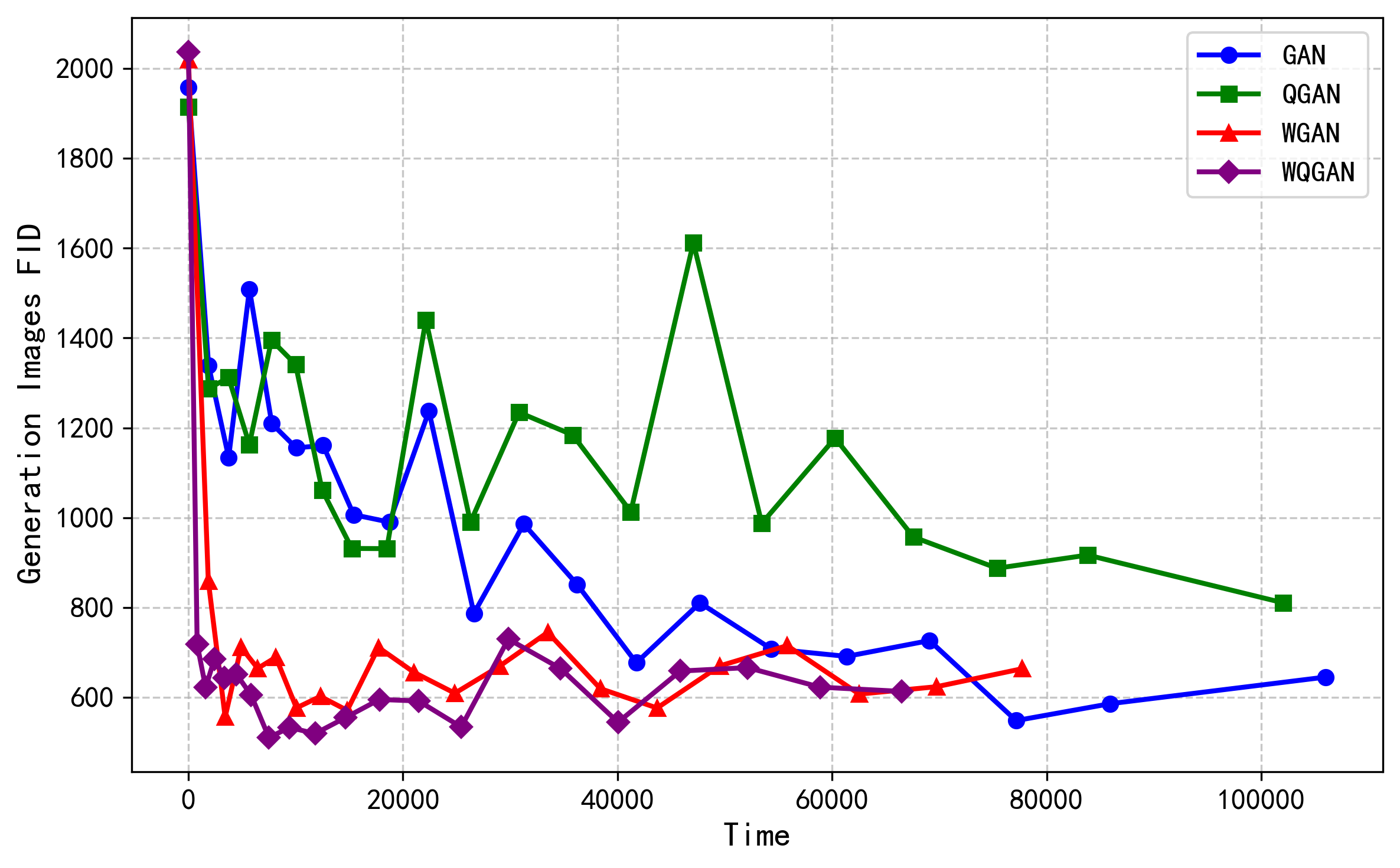}
	\caption{Line graph of FID for generated images (CelebA).}
	\label{p:celeba_fid}
\end{figure}

\begin{figure}[ht]
	\centering
	
	\setlength{\abovecaptionskip}{-0cm}
	\setlength{\belowcaptionskip}{-0cm}
	\vspace{-0.1cm}
	\includegraphics[width=0.9\textwidth,height=0.45\textwidth]{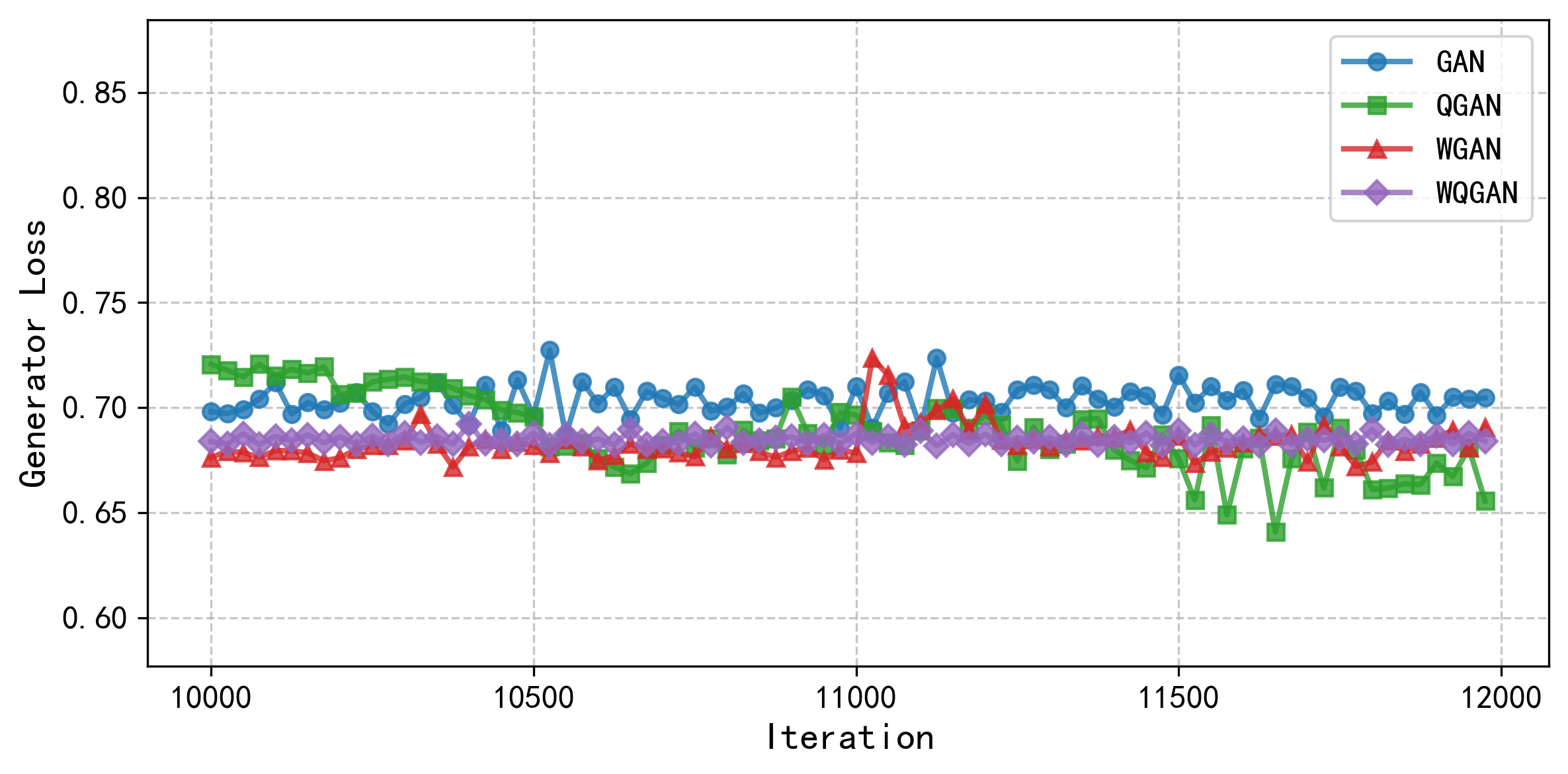}
	\caption{Line graph of generation losses for images (CELEBA).}
	\label{p:celeba_loss_com}
\end{figure}

All experiments were conducted on a personal computer with CPU: 12th GEN intel i5-12400 and GPU: NVIDIA GeForce 1070Ti 8G.

We test the quality of images generated by various algorithms at the initial iteration. Due to the need to call the classification model Inception v3 for each evaluation, which consumes a portion of the GPU memory, we only assess the generation quality of the four generative models for the first $20$k iterations of generating images at $2$k intervals. We list the FID and IS metrics about SVHN in Table \ref{t:SVHN} and CelebA in Table \ref{t:Celeba}. Figure \ref{p:svhn_fid_com} and Figure \ref{p:celeba_fid} show the trends of FID with respect to time.
To better compare the convergence states, Figure \ref{p:svhn_loss_com} and Figure \ref{p:celeba_loss_com} show the average losses calculated every 50 iterations between the 10k and 12k iterations.
We can see that WQGAN can converge to the real image faster compared to other models. GAN model, reaching the optimal metrics at some iterations, is unable to maintain the advantage and constantly oscillates. WGAN with introduction of Wasserstein distance and QGAN with introduction of quaternion deconvolution both improve the problem of GAN instability, but it is all one-sided improvement. WQGAN combines the advantages of Wasserstein distance and quaternion deconvolution, which enable the generative model to generate high-quality images faster and achieve better generation results.

\begin{figure}[ht]
	\centering
	\setlength{\abovecaptionskip}{-0cm}
	\setlength{\belowcaptionskip}{-0cm}
	\vspace{-0.1cm}
	\includegraphics[width=1\textwidth,height=0.7\textwidth]{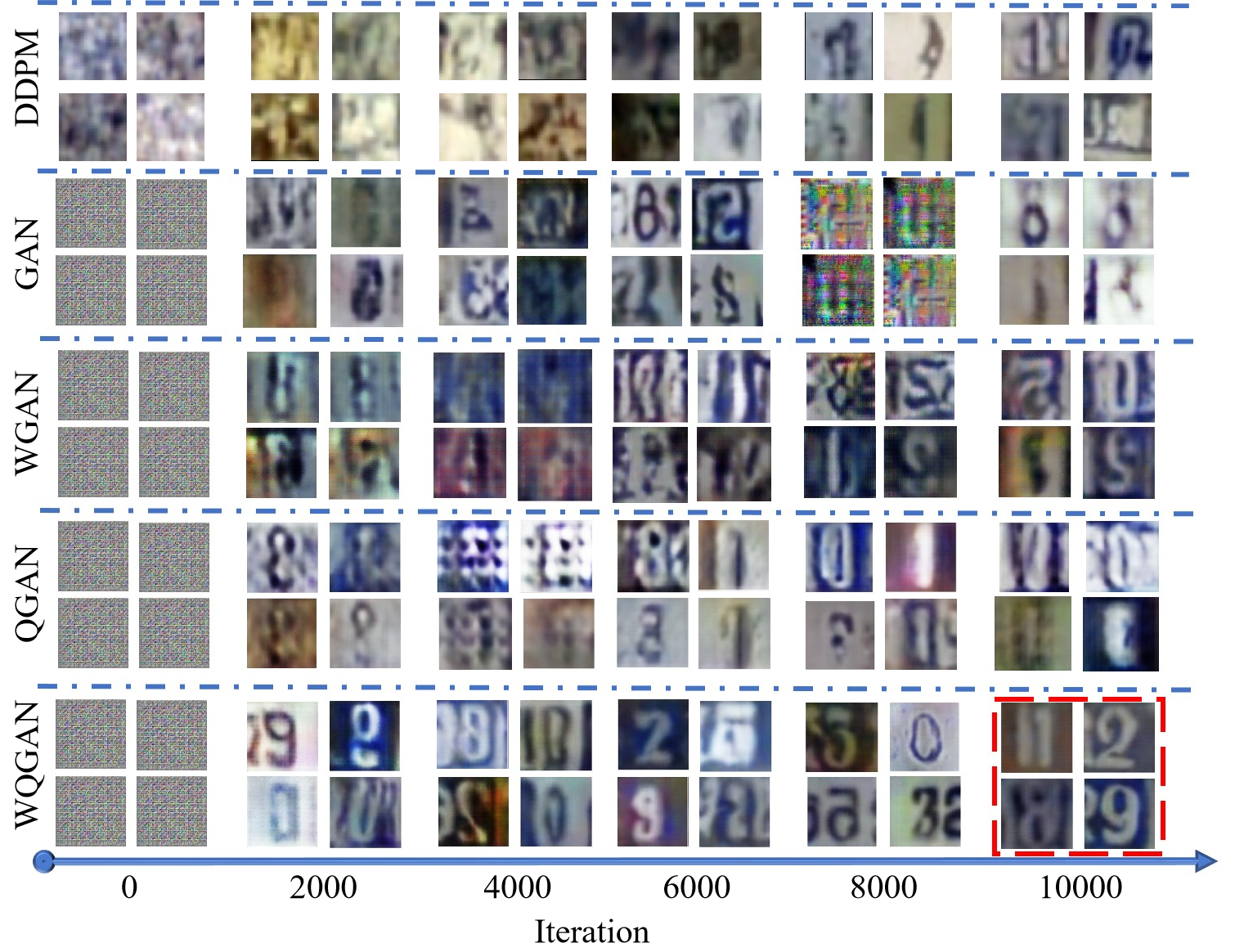}
	\caption{Examples of SVHN images generated by the four generative models at iterations from 0 to 10k. Four images are displayed each 2k iterations. From top to bottom: DDPM, GAN, WGAN, QGAN, WQGAN.}
	\label{p:trend_svhn}
\end{figure}

For the generation process of those generative models, we will showcase it using the SVHN dataset. Figure \ref{p:trend_svhn} illustrates the evolution of the generated images by the five models over the first 10k iterations. The images are sampled every 2k iterations, displaying four images at a time. As we observe, the images generated by the GAN and DDPM models are merely combinations of lines, and there is a phenomenon of mode collapse during the training process. The images generated at 8k iterations are almost akin to colored noise, making them difficult to discern. This can lead to challenges or even failure in later stages of training. WGAN can generate some street scene digital photos, but there are still issues with the generated fonts being distorted and difficult to identify, as well as problems with the blending of font colors with background colors, and deviations. These circumstances result in images that are not sufficiently clean and aesthetically pleasing, creating disparities with real datasets. WQGAN can promptly identify the generation task and initiate stable generation of street scene digital images. As the number of iterations increases, the generated street scene images progressively approach those in real datasets. From Figure \ref{p:trend_svhn}, we can observe that the image quality generated at 10k iterations is significantly superior to those generated by other models. This demonstrates that WQGAN combines the strengths of WGAN and QGAN, enabling not only rapid generation of target images but also an awareness of relationships between channels, thereby avoiding color deviation issues.

\begin{table}[ht]
	\centering

	\begin{tabular}{llllllll}
		\hline
		\multicolumn{8}{l}{\textbf{DDPM : FID = 72.3748 IS = 1.75 $\pm$ 0.55}} \\ \hline\hline
		\begin{minipage}[b]{0.09\columnwidth}
			\centering
			\vspace{1pt} 
			\raisebox{-.5\height}{\includegraphics[width=\linewidth]{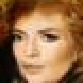}}
		\end{minipage}
		& \begin{minipage}[b]{0.09\columnwidth}
			\centering
			\raisebox{-.5\height}{\includegraphics[width=\linewidth]{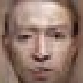}}
		\end{minipage}        
		& \begin{minipage}[b]{0.09\columnwidth}
			\centering
			\raisebox{-.5\height}{\includegraphics[width=\linewidth]{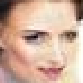}}
		\end{minipage}        
		& \begin{minipage}[b]{0.09\columnwidth}
			\centering
			\raisebox{-.5\height}{\includegraphics[width=\linewidth]{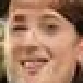}}
		\end{minipage}         
		& \begin{minipage}[b]{0.09\columnwidth}
			\centering
			\raisebox{-.5\height}{\includegraphics[width=\linewidth]{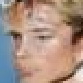}}
		\end{minipage}                  
		&\begin{minipage}[b]{0.09\columnwidth}
			\centering
			\raisebox{-.5\height}{\includegraphics[width=\linewidth]{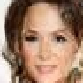}}
		\end{minipage}       
		& \begin{minipage}[b]{0.09\columnwidth}
			\centering
			\raisebox{-.5\height}{\includegraphics[width=\linewidth]{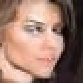}}
		\end{minipage}         
		& \begin{minipage}[b]{0.09\columnwidth}
			\centering
			\raisebox{-.5\height}{\includegraphics[width=\linewidth]{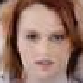}}
		\end{minipage}               \\ \hline
		\multicolumn{8}{l}{\textbf{GAN : FID = 68.1870 IS = 1.66 $\pm$ 0.24}} \\ \hline\hline
		\begin{minipage}[b]{0.09\columnwidth}
			\centering
			\vspace{1pt} 
			\raisebox{-.5\height}{\includegraphics[width=\linewidth]{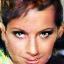}}
		\end{minipage}
		& \begin{minipage}[b]{0.09\columnwidth}
			\centering
			\raisebox{-.5\height}{\includegraphics[width=\linewidth]{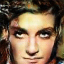}}
		\end{minipage}        
		& \begin{minipage}[b]{0.09\columnwidth}
			\centering
			\raisebox{-.5\height}{\includegraphics[width=\linewidth]{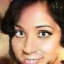}}
		\end{minipage}        
		& \begin{minipage}[b]{0.09\columnwidth}
			\centering
			\raisebox{-.5\height}{\includegraphics[width=\linewidth]{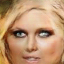}}
		\end{minipage}         
		& \begin{minipage}[b]{0.09\columnwidth}
			\centering
			\raisebox{-.5\height}{\includegraphics[width=\linewidth]{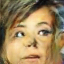}}
		\end{minipage}                  
		&\begin{minipage}[b]{0.09\columnwidth}
			\centering
			\raisebox{-.5\height}{\includegraphics[width=\linewidth]{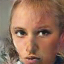}}
		\end{minipage}       
		& \begin{minipage}[b]{0.09\columnwidth}
			\centering
			\raisebox{-.5\height}{\includegraphics[width=\linewidth]{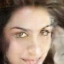}}
		\end{minipage}         
		& \begin{minipage}[b]{0.09\columnwidth}
			\centering
			\raisebox{-.5\height}{\includegraphics[width=\linewidth]{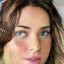}}
		\end{minipage}               \\ \hline
		
		\multicolumn{8}{l}{\textbf{QGAN : FID =   80.5843 IS = 1.67 $\pm$ 0.25}} \\ \hline\hline
		\begin{minipage}[b]{0.09\columnwidth}
			\centering
			\vspace{1pt} 
			\raisebox{-.5\height}{\includegraphics[width=\linewidth]{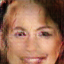}}
		\end{minipage}
		& \begin{minipage}[b]{0.09\columnwidth}
			\centering
			\raisebox{-.5\height}{\includegraphics[width=\linewidth]{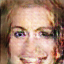}}
		\end{minipage}        
		& \begin{minipage}[b]{0.09\columnwidth}
			\centering
			\raisebox{-.5\height}{\includegraphics[width=\linewidth]{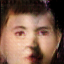}}
		\end{minipage}        
		& \begin{minipage}[b]{0.09\columnwidth}
			\centering
			\raisebox{-.5\height}{\includegraphics[width=\linewidth]{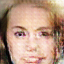}}
		\end{minipage}         
		& \begin{minipage}[b]{0.09\columnwidth}
			\centering
			\raisebox{-.5\height}{\includegraphics[width=\linewidth]{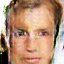}}
		\end{minipage}                  
		&\begin{minipage}[b]{0.09\columnwidth}
			\centering
			\raisebox{-.5\height}{\includegraphics[width=\linewidth]{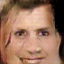}}
		\end{minipage}        
		& \begin{minipage}[b]{0.09\columnwidth}
			\centering
			\raisebox{-.5\height}{\includegraphics[width=\linewidth]{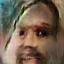}}
		\end{minipage}         
		& \begin{minipage}[b]{0.09\columnwidth}
			\centering
			\raisebox{-.5\height}{\includegraphics[width=\linewidth]{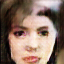}}
		\end{minipage}               \\ \hline
		
		\multicolumn{8}{l}{\textbf{WGAN : FID =   88.4351 IS = 1.71 $\pm$  0.26}} \\ \hline\hline
		\begin{minipage}[b]{0.09\columnwidth}
			\centering
			\vspace{1pt} 
			\raisebox{-.5\height}{\includegraphics[width=\linewidth]{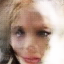}}		
		\end{minipage}
		& \begin{minipage}[b]{0.09\columnwidth}
			\centering
			\raisebox{-.5\height}{\includegraphics[width=\linewidth]{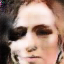}}
		\end{minipage}        
		& \begin{minipage}[b]{0.09\columnwidth}
			\centering
			\raisebox{-.5\height}{\includegraphics[width=\linewidth]{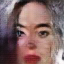}}
		\end{minipage}        
		& \begin{minipage}[b]{0.09\columnwidth}
			\centering
			\raisebox{-.5\height}{\includegraphics[width=\linewidth]{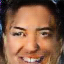}}
		\end{minipage}         
		& \begin{minipage}[b]{0.09\columnwidth}
			\centering
			\raisebox{-.5\height}{\includegraphics[width=\linewidth]{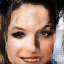}}
		\end{minipage}                  
		&\begin{minipage}[b]{0.09\columnwidth}
			\centering
			\raisebox{-.5\height}{\includegraphics[width=\linewidth]{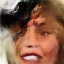}}
		\end{minipage}       
		& \begin{minipage}[b]{0.09\columnwidth}
			\centering
			\raisebox{-.5\height}{\includegraphics[width=\linewidth]{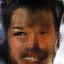}}
		\end{minipage}         
		& \begin{minipage}[b]{0.09\columnwidth}
			\centering
			\raisebox{-.5\height}{\includegraphics[width=\linewidth]{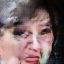}}
		\end{minipage}   \\ \hline
		\multicolumn{8}{l}{\textbf{WQGAN : FID = 68.3037 IS = 1.73 $\pm$  0.18}}  \\ \hline\hline
		\begin{minipage}[b]{0.09\columnwidth}
			\centering
			\vspace{1pt} 
			\raisebox{-.5\height}{\includegraphics[width=\linewidth]{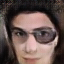}}		
		\end{minipage}
		& \begin{minipage}[b]{0.09\columnwidth}
			\centering
			\raisebox{-.5\height}{\includegraphics[width=\linewidth]{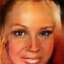}}
		\end{minipage}               
		& \begin{minipage}[b]{0.09\columnwidth}
			\centering
			\raisebox{-.5\height}{\includegraphics[width=\linewidth]{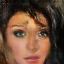}}
		\end{minipage}         
		& \begin{minipage}[b]{0.09\columnwidth}
			\centering
			\raisebox{-.5\height}{\includegraphics[width=\linewidth]{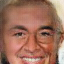}}
		\end{minipage}                  
		&\begin{minipage}[b]{0.09\columnwidth}
			\centering
			\raisebox{-.5\height}{\includegraphics[width=\linewidth]{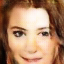}}
		\end{minipage}
		& \begin{minipage}[b]{0.09\columnwidth}
			\centering
			\raisebox{-.5\height}{\includegraphics[width=\linewidth]{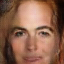}}
		\end{minipage}              
		& \begin{minipage}[b]{0.09\columnwidth}
			\centering
			\raisebox{-.5\height}{\includegraphics[width=\linewidth]{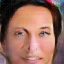}}
		\end{minipage}         
		& \begin{minipage}[b]{0.09\columnwidth}
			\centering
			\raisebox{-.5\height}{\includegraphics[width=\linewidth]{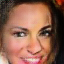}}
		\end{minipage}
		
		\\ \hline  
	\end{tabular}
	\caption{Sample images from the three generative models at 50k iterations.}
	\label{t:com}
\end{table}

\begin{figure}[ht]
	\centering
	\setlength{\abovecaptionskip}{-0cm}
	\setlength{\belowcaptionskip}{-0cm}
	\vspace{-0.1cm}
	\includegraphics[width=1\textwidth,height=0.55\textwidth]{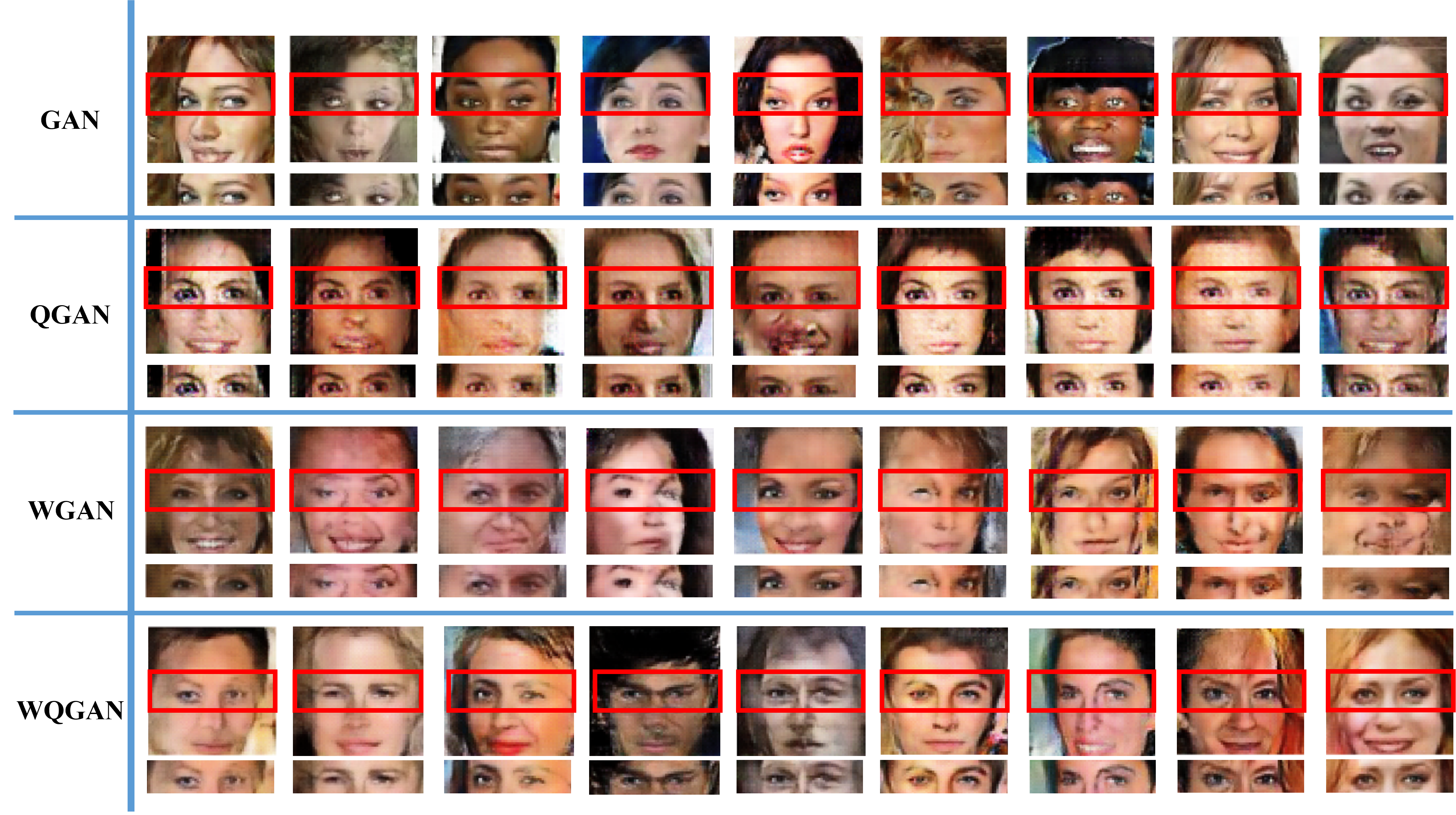}
	\caption{The diversity of four generative models.}
	\label{p:similarly_com}
\end{figure}

Generating human face images is more complex compared to street scene images, requiring generative models to have a longer training time to learn how to generate them. We subsequently test the image generation quality of WQGAN, WGAN, QGAN, GAN and DDPM at 50,000 iterations. Table \ref{t:com} shows that WQGAN possesses superior generative performance. WQGAN can generate diverse facial images, aiming to depict a variety of features for each facial component. For instance, the mouth may exhibit expressions like smiling or pursing, and the eyes could range from squinting to wide open, or even wearing glasses. WGAN only captures the basic shapes of facial features, lacking detailed depiction. This indicates that WGAN has not been trained sufficiently at this iteration count and further training is required. QGAN also begins to sketch facial images. However, an issue arises with the consistent trend in the depiction of facial features, particularly the eyes, as shown in Figure \ref{p:similarly_com}. In the sample images, the eyes generated by QGAN appear visually similar across faces. This suggests that QGAN may lead to a lack of diversity in generated images.

Above all, WQGAN combines the strengths of both WGAN and QGAN as Table \ref{t:com_gans} shows. The introduction of QWD further advances QGAN and holds promise for its application in a broader range of quaternion neural networks.

\begin{table}[ht]
	\centering
	
	\begin{tabular}{|ccccc|}
		\hline
		Relative   Performances & GAN & QGAN & WGAN & WQGAN \\ \hline
		Stability               &     & \checkmark    & \checkmark    & \checkmark     \\
		Siversity               &     &      & \checkmark    & \checkmark     \\
		Speed                   &     &      & \checkmark    & \checkmark     \\
		Quality              &     & \checkmark     &      & \checkmark     \\ \hline
	
	\end{tabular}
	\vspace{0.2cm}
	\caption{Comparison of four generative adversarial networks.}
\label{t:com_gans}
\end{table}

\section{Conclusion} \label{sec:con}
In this paper, we introduce a novel QWD and derive the corresponding dual form by introducing a new quaternion linear programming problem. Subsequently, we propose a new WQGAN by changing training objective function of QGAN. This novel model not only effectively utilizes the interplay among three channels in the color space but also evaluates the discrepancy between generated data and real data through Wasserstein distance, leading to faster generation of high-quality images.

In future, we will optimize this model to enable it to generate higher-quality  images and apply it to various practical applications of color image processing.

\section*{Acknowledgments}
The authors are grateful to the editor and the anonymous referees for their wonderful comments and helpful suggestions. We are also grateful to Dong Yao and Xiao-Bin Sun from JSNU for their excellent comments and suggestions on the definition and properties of QWD.









\end{document}